\documentclass{article}

\usepackage[preprint]{neurips_2026}

\usepackage{algorithm}
\usepackage{multirow}
\usepackage{afterpage}
\usepackage{tabularx}
\usepackage{wrapfig}
\usepackage{capt-of}
\usepackage[table]{xcolor}   
\usepackage{booktabs}        
\usepackage[most]{tcolorbox}
\tcbuselibrary{skins}
\newtcolorbox{rephrasebox}[1]{
  enhanced, colback=white, colframe=black!30,
  colbacktitle=black!5, coltitle=black, fonttitle=\small\bfseries,
  title={#1}, boxrule=0.5pt, arc=2pt,
  sidebyside, sidebyside align=top seam, sidebyside gap=14pt,
  left=6pt, right=6pt, top=5pt, bottom=5pt}
\usepackage{algpseudocode}
\newsavebox{\tabonebox}
\usepackage[utf8]{inputenc} 
\usepackage[T1]{fontenc}    
\usepackage{url}            
\usepackage{booktabs}       
\definecolor{cqwen}{RGB}{123,80,162}    
\definecolor{cllama}{RGB}{0,114,178}    
\definecolor{ccomas}{RGB}{205,60,66}    
\definecolor{corange}{RGB}{232,126,42}  
\definecolor{cgain}{RGB}{20,120,50}     
\definecolor{cdrop}{RGB}{190,55,55}         
\usepackage{graphicx} 

\usepackage{amsmath}

\usepackage{amsfonts}       
\usepackage{nicefrac}       
\usepackage{microtype}      
\usepackage{xcolor}         
\usepackage{bbm}
\definecolor{citeblue}{RGB}{0,102,204}
\definecolor{refred}{RGB}{216, 81, 64}
\usepackage{amsthm}
\usepackage{xspace}

\usepackage[colorlinks=true,
            linkcolor=refred,
            citecolor=citeblue,
            urlcolor=blue,
            hyperfootnotes=false]{hyperref}
\newcommand{\method}{\textsc{Co-RL}\xspace}
\usepackage{amsmath}
\newcommand{\eqsize}{\small}
\newlength{\eqskip}
\newcommand{\eqsetup}{\eqsize\setlength{\abovedisplayskip}{\eqskip}\setlength{\belowdisplayskip}{\eqskip}\setlength{\abovedisplayshortskip}{0pt}\setlength{\belowdisplayshortskip}{\eqskip}}
\usepackage{amssymb}
\usepackage{amsthm}
\usepackage{bbm}
\newtheorem{proposition}{Proposition}
\newtheorem{theorem}{Theorem}

\definecolor{takeawaygreen}{RGB}{82,128,57}

\newtcolorbox{takeaway}{
  enhanced,
  blanker,             
  breakable,
  left=4pt,
  right=4pt,
  top=6pt,
  bottom=6pt,
  borderline north={1.2pt}{0pt}{takeawaygreen},
  borderline south={1.2pt}{0pt}{takeawaygreen},
}
\title{Co-RL: Unsupervised Reasoning Emerges from Diverse Cohort in Multi-agent RL}

\author{%
  \textbf{Yunhao Yang}$^{* \spadesuit}$\quad
  \textbf{Yuexin Bian}$^{* \heartsuit}$\quad
  \textbf{Yunjie Tian}$^{\clubsuit}$\quad
  \textbf{Di Fu}$^{\clubsuit}$\quad
  \textbf{Tianjin Huang}$^{\diamondsuit}$\\
  \textbf{Yuanyuan Shi}$^{\heartsuit}$\quad
  \textbf{Ziang Xiao}$^{\spadesuit}$\quad
  \textbf{Nuno Vasconcelos}$^{\heartsuit}$\quad
  \textbf{Yijiang Li}$^{\heartsuit \dagger}$
  \\[4pt]
  {\normalfont
    $^{\spadesuit}$Johns Hopkins University \quad
    $^{\heartsuit}$UC San Diego}\\
  {\normalfont
    $^{\diamondsuit}$University of Exeter \quad
    $^{\clubsuit}$ByteDance}\\[3pt]
  {\normalfont\texttt{\{yijiangli, nuno\}@ucsd.edu}}
}

\begin{document}

\maketitle

\renewcommand{\thefootnote}{}
\footnotetext{$^{*}$Equal contribution. $^{\dagger}$Project lead.}
\footnotetext{
Corresponding to Yijiang Li: yijiangli@ucsd.edu, Nuno Vasconcelos: nuno@ucsd.edu}
\renewcommand{\thefootnote}{\arabic{footnote}}
\begin{abstract}
Reinforcement Learning (RL) has emerged as a powerful approach for improving reasoning in language and vision-language models, yet its strongest successes still depend heavily on ground-truth supervision (e.g., verifiable reward). Such annotations are costly to obtain and become increasingly scarce as reasoning capabilities advance beyond what humans can reliably evaluate. 
Self-rewarding RL reduces this dependence by enabling models to derive reward signals from their own completions. However, training solely on self-generated feedback can reinforce existing biases, reduce response diversity, and ultimately lead to homogenized responses and training collapse. 
In this work, we show that unsupervised reasoning can emerge through cooperative multi-agent training. We introduce \method, a framework in which multiple decoupled models, sharing no parameters, are simultaneously optimized through RL using rewards derived from the other agent. We further show that increasing cohort diversity, through heterogeneous model families, sizes, and rephrased training samples, reduces the correlated errors that drive self-reinforcing feedback loops. This diversity consistently improves reasoning performance and mitigates training collapse.
Without any ground-truth labels, \method improves the base models by 3.0--7.9\% on average across seven benchmarks and by 1.6--4.4\% across four multimodal benchmarks. It also outperforms the prior state-of-the-art multi-agent RL method CoMAS by 4.0\% and the self-rewarding vision-language RL method MM-UPT by 2.0\%. {\urlstyle{same}Our code is available at \url{https://github.com/Diverse-CoRL/Co-RL} and website at \url{https://williamium3000.github.io/Co-RL}.}

\end{abstract}

\section{Introduction}
\label{sec:intro}


Reinforcement learning with verifiable rewards (RLVR) has emerged as a powerful approach for improving reasoning in large language models \citep{lightman2023letsverify, guo2025deepseekr1}, yet its strongest successes still depend heavily on ground-truth supervision. Such supervision is costly to obtain and becomes increasingly scarce as target reasoning capabilities approach or surpass what humans can reliably evaluate \citep{yue2025rlvr}. Self-rewarding RL reduces this dependence by deriving rewards from the model's own completions, incorporating signals such as agreement with its majority-vote prediction \citep{zuo2025ttrl, li2026policy}, self-certainty \citep{intuitor2025}, predictive entropy \citep{rent2025}, or consistency across paraphrased inputs or moving-average policies \citep{corewarding2025}. However, these signals remain within a single model's own predictions. Without an external reference, such self-reinforcement can amplify existing biases and suboptimal behaviors, reduce response diversity, and ultimately lead to increasingly homogeneous outputs or even training collapse.


This raises a fundamental question: \emph{how can a model obtain a sufficiently independent learning signal to improve without any ground-truth supervision?} In this work, we show that such a signal can emerge from independently trained models. Since their errors are not perfectly correlated, each model can provide corrective feedback that the other cannot derive from its own generations. We therefore take the learning signal from a separate model, giving each agent \emph{decorrelated supervision}: a target produced by independently updated weights that is less likely to echo its own biases \citep{blum1998cotraining, li2023diverse}.

\begin{figure}[t]
    \centering
    \includegraphics[width=1\linewidth]{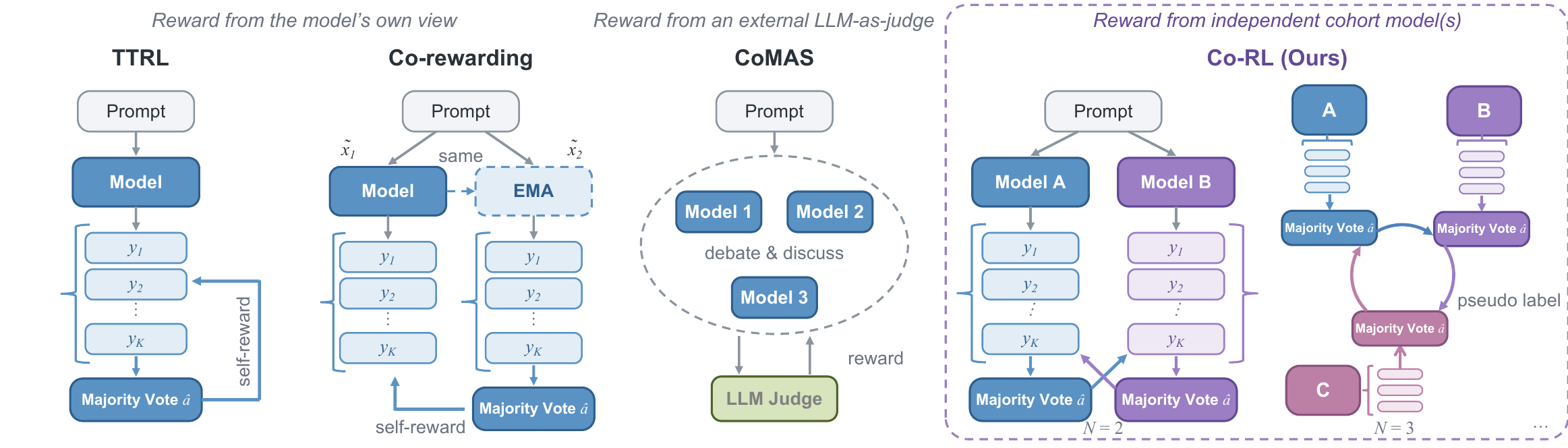}
    \caption{\textbf{Comparison of \method with prior label-free RL methods.} TTRL and Co-rewarding derive rewards from self-generated majority-vote responses, while CoMAS evaluates multi-turn interactions using one of its own agents as an LLM judge. In contrast, \method derives rewards from the votes of an independently trained model, decoupling the reward signal from the agent being optimized.
}
    \label{fig:comp_ours_prior}
\end{figure}

Building on this insight, we introduce \method, a cooperative multi-agent label-free RL method in which multiple decoupled models, sharing no parameters, are optimized simultaneously using rewards derived from their cohorts. Given an unlabeled prompt, each agent samples multiple completions and aggregates them into a pseudo-label through majority voting~\citep{wang2023selfconsistency, li2026policy}. The completions of one agent are then rewarded against another agent's pseudo-label, and the resulting rewards drive policy optimization with GRPO~\citep{shao2024deepseekmath}. Unlike self-rewarding methods, where the policy and the reward come from the same model, \method draws supervision from an independently updated model, weakening the direct feedback loop that amplifies a model's own bias. Cohort diversity is essential to \method, as highly similar agents are more likely to make correlated errors, increasing the risk of reinforcing the same incorrect predictions. To promote complementary behaviors, we diversify agents along several dimensions, including model architecture and pretrained weights, model scale, and input formulation (e.g., rephrased prompts~\citep{corewarding2025}). These sources of heterogeneity expose agents to distinct inductive biases and decision boundaries, reducing error correlation and strengthening the corrective signals available within the cohort. 

As illustrated in Figure~\ref{fig:comp_ours_prior}, \method differs fundamentally from both self-rewarding and multi-agent RL approaches. Whereas self-rewarding methods derive rewards from an agent's own predictions, \method obtains supervision from an independently trained model, decoupling reward generation from the agent being optimized. Unlike prior multi-agent RL methods, it further requires no explicit inter-agent interaction beyond reward construction and no external LLM judge or learned reward model~\citep{comas2025,maporl2025,marti2026}. This yields a lightweight and symmetric framework in which every agent serves both as a learner and as a source of supervision.

We benchmark \method across diverse model families in both text-only and multimodal domains. Across seven text-only benchmarks spanning mathematical reasoning, coding, and knowledge-intensive reasoning, \method improves the base models by 3.0--7.9\% on average and outperforms the strongest label-free baselines by 0.8--2.0\%. The gains extend to multimodal reasoning, where \method improves vision-language models (VLMs) by 1.6--4.4\% across four benchmarks. \method also surpasses the prior best multi-agent method such as CoMAS~\citep{comas2025} by 4.0\% and the multimodal self-rewarding method MM-UPT~\citep{mmupt2025} by 2.0\%.

\section{Related Work}
\label{app:related}
\textbf{Self-rewarding RL.}
While RLVR effectively improves LLM reasoning~\citep{shao2024deepseekmath, guo2025deepseekr1, yu2025dapo}, its reliance on high-quality ground-truth labels remains a major bottleneck~\citep{yue2025rlvr}. Recent work therefore explores self-rewarding mechanisms that learn from unlabeled data. One prominent direction derives reward signals solely from a model's own behavior, such as majority voting~\citep{zuo2025ttrl}, self-consistency~\citep{wang2023selfconsistency}, self-certainty~\citep{intuitor2025, li2025rlsc}, or predictive entropy~\citep{rent2025, zhang2025empo}. Because these methods rely on a single model's own view, confident errors may be reinforced without corrective feedback.
This line extends earlier work on self-rewarding models \citep{yuan2024selfrewarding} and self-play supervision \citep{chen2024spin}, and now includes unsupervised self-training \citep{xu2025genius, fang2025serl}, self-correction based on a model's own judgments \citep{xiong2025selfrewardingcorrection}, and zero-data self-evolution, where one or more models generate their own curricula \citep{zhao2025absolutezero, huang2025rzero, liu2025spiral}. 

\textbf{Co-training and cross-view supervision.}
Using a second view to supervise a learner is a longstanding idea. Co-training shows that two conditionally independent views can teach each other~\citep{blum1998cotraining}, while subsequent analysis demonstrates that highly similar views tend to reinforce the same errors: the more alike the two views are, the more they simply agree on the same mistakes~\citep{li2023diverse}. Related principles underlie self-supervised representation learning, where augmentations, momentum targets, or stop-gradient operations prevent collapse~\citep{chen2020simclr,grill2020byol,singh2024shifting,caron2021dino}, and deep mutual learning, where peer networks provide reciprocal supervision~\citep{zhang2018mutual}. Co-rewarding~\citep{corewarding2025} derives rewards from paraphrased questions or a slowly updated model copy. \method instead uses separately optimized models and directly studies the role of cross-family diversity.

\textbf{Multi-agent RL.}
Several methods improve reasoning by enabling interactions among multiple models. At inference time, multi-agent debate and round-table consensus have models critique and revise one another~\citep{du2023debate,liang2023encouraging,chen2023reconcile,sun2023corex}, and general orchestration frameworks compose such agents into pipelines~\citep{wu2023autogen,zhang2024chainofagents,kim2024mdagents}. A more recent line trains the agents: CoMAS turns interactions scored by an LLM judge into rewards~\citep{comas2025}, while MAPoRL and MARFT co-train agents against a learned or shared reward~\citep{maporl2025,liao2025marft}. \method couples agents through their respective votes, requiring neither multi-turn communication nor an external reward judge.

\textbf{RLVR for multimodal reasoning.}
RLVR has rapidly expanded to vision-language models, with R1-style rule-based rewards underpinning a broad family of methods~\citep{visionr12025,shen2025vlmr1,feng2025videoorion,liu2025visualrft,meng2025mmeureka,peng2025lmmr1,li2024core,li2026self, yang2025r1onevision,wang2025skyworkr1v2}. Existing work largely assumes verifiable labels, while label-free multimodal RL remains underexplored and relies on a single model's own signals~\citep{mmupt2025}. Cross-view supervision is especially promising here because VLM families combine different vision encoders and language backbones.

\section{Preliminary} 
\label{sec:preliminary}
\subsection{Reinforcement Learning for Reasoning}
\label{sec:prelim_rl_reasoning}

Given a prompt $x \sim \mathcal{D}$, an autoregressive language model
$\pi_\theta$ generates a response $y=(y_1,\ldots,y_T)$ according to $
    \pi_\theta(y\mid x)
    =
    \prod_{t=1}^{T}
    \pi_\theta(y_t\mid x,y_{<t})$.
Reinforcement learning optimizes the policy using a scalar reward $r(x,y)$
that evaluates the quality of the generated response. The corresponding
objective is
{\eqsetup
\begin{equation}
    \mathcal{J}(\theta)
    =
    \mathbb{E}_{x\sim\mathcal{D},\,
    y\sim\pi_\theta(\cdot\mid x)}
    \left[r(x,y)\right].
\end{equation}}%
For reasoning tasks, the response $y$ typically contains both an intermediate
reasoning trajectory and a final answer, while the reward is commonly assigned
at the response level based on answer correctness or another outcome-based
criterion.

Among policy optimization algorithms, Group Relative Policy Optimization (GRPO) has been widely adopted as it improves reasoning performance without requiring a separately trained critic model~\citep{shao2024deepseekmath,guo2025deepseekr1}. Given a prompt $x$, the old policy $\pi_{\theta_{\mathrm{old}}}$ samples a group of $K$ responses $\{y^1,\ldots,y^K\}$ and assigns each response a reward $r^k=r(x,y^k)$. GRPO estimates the advantage of each response by normalizing rewards within the group: $ \hat{A}^k = \frac{ r^k-\operatorname{mean}(\{r^j\}_{j=1}^{K}) }{ \operatorname{std}(\{r^j\}_{j=1}^{K})+\epsilon }$, where $\epsilon$ is a small constant for numerical stability. Let $ \rho_{k,t}(\theta) = \frac{ \pi_\theta(y_t^k\mid x,y_{<t}^k) }{ \pi_{\theta_{\mathrm{old}}}(y_t^k\mid x,y_{<t}^k) }$ denote the token-level importance ratio. GRPO optimizes the clipped surrogate objective {\eqsetup\begin{equation}\label{eq:grpo}
\mathcal{J}_{\mathrm{GRPO}}(\theta) = \mathbb{E}\!\left[ \frac{1}{K} \sum_{k=1}^{K} \frac{1}{|y^k|} \sum_{t=1}^{|y^k|} \left[ \min\!\left( \rho_{k,t}(\theta)\hat{A}^k, \operatorname{clip}\!\left( \rho_{k,t}(\theta), 1-\delta, 1+\delta \right)\hat{A}^k \right) - \beta\mathcal{D}_{\mathrm{KL}}^{k,t} \right] \right], \end{equation}}%
where $\delta$ is the clipping threshold, $\beta$ controls the strength of KL regularization, and $\mathcal{D}_{\mathrm{KL}}^{k,t}$ penalizes deviation from a reference policy $\pi_{\mathrm{ref}}$. GRPO typically uses verifiable rewards from ground-truth answers or external verifiers. However, such supervision can be costly or difficult to obtain at scale. The following subsection considers model-generated rewards as an alternative, enabling reinforcement learning on unlabeled prompts.

\subsection{Self-rewarding RL}
\label{sec:prelim_self_rewarding}
In the absence of external verifiers, reward functions can be constructed directly from a model's own outputs, thereby enabling RL on unlabeled prompts. Given an unlabeled prompt $x$, the policy $\pi_\theta$ samples responses $\{y^1,\ldots,y^K\}$, with answers $a^k=g(y^k)$ extracted from each response. TTRL~\citep{zuo2025ttrl} defines the reward for response $y^k$ as $r_{\mathrm{maj}}^k=\mathbf{1}[a^k=\hat a_\theta(x)]$, where $\hat a_\theta(x)=\arg\max_b \sum_{j=1}^{K}\mathbf{1}[a^j=b]$ is the policy's majority-vote answer. Intuitor~\citep{intuitor2025} instead measures token-level confidence using $r_{\mathrm{conf}}^k=\frac{1}{|y^k|}\sum_{t=1}^{|y^k|}D_{\mathrm{KL}}(\mathbf{U}\Vert\pi_\theta(\cdot\mid x,y_{<t}^k))$, where $\mathbf{U}$ denotes the uniform distribution over the vocabulary. RENT~\citep{rent2025} operates on the predictive distribution, assigning $r_{\mathrm{ent}}^k=-\frac{1}{|y^k|}\sum_{t=1}^{|y^k|}\mathcal{H}(\pi_\theta(\cdot\mid x,y_{<t}^k))$, where $\mathcal{H}(\cdot)$ denotes entropy. 
Thus, TTRL~\citep{zuo2025ttrl} rewards responses that agree with the policy's majority-vote prediction, whereas Intuitor~\citep{intuitor2025} and RENT~\citep{rent2025} derive rewards from the model's token-level predictive distributions.
Despite their different reward constructions, these methods ultimately produce scalar rewards that can be used for policy optimization. Under GRPO, the self-generated rewards $\{r^k\}_{k=1}^K$ are normalized into group-relative advantages $\hat A^k$ and used directly in Eq.~\eqref{eq:grpo}.

Despite their effectiveness, these signals originate from the same policy being optimized: without an external reference, training may reinforce existing biases and suboptimal behaviors, reduce response diversity, and eventually lead to homogenized responses or training collapse. We show this in (b) and (c) of Figure \ref{fig:method_panels}, where prolonged training causes TTRL to degenerate and lead to training collapse.

\section{Method}
\label{sec:method}
\subsection{Co-Reinforcement Learning (Co-RL)}
\label{sec:colearn}

In this work, we introduce Co-Reinforcement Learning (\method), a label-free multi-agent reinforcement learning framework in which multiple decoupled agents independently solve unlabeled problems and provide supervision to one another through majority-vote pseudo-rewards. The agents share neither parameters nor gradients, and their optimization is coupled solely through the reward signals. 
Formally, let $x\sim\mathcal D$ denote an unlabeled reasoning problem. We consider a cohort of $N$ independent agents parameterized by $\{\pi_{\theta_n}\}_{n=1}^{N}$. The agents may be initialized from the same pretrained model or from models of different families and sizes. For each problem $x$, agent $n$ independently rolls out a group of $K$ completions $y_n^k\overset{\mathrm{i.i.d.}}{\sim}\pi_{\theta_n}(\cdot\mid x)$, $k=1,\ldots,K$, and extracts the corresponding final answers as $a_n^k=g(y_n^k)$, with $g(\cdot)$ denoting an answer-extraction function.

For each agent $n$, \method constructs a supervision target exclusively from the answers generated by one designated model. Specifically, the pseudo-label and the resulting reward for agent \(n\) are defined as $\hat a_{-n}(x)\in\arg\max_b\sum_{j=1}^{K}\mathbf1[a_{n-1}^j=b]$ and $r_n^k=\mathbf1[a_n^k=\hat a_{-n}(x)]$, where $b$ ranges over the answers produced by agent $n-1$, with the index taken cyclically so that agent $1$ is supervised by agent $N$. Thus, $\hat a_{-n}(x)$ represents the majority-vote answer of the other agent that supervises agent $n$. A response receives reward $1$ when its extracted answer agrees with the constructed pseudo-label, and reward $0$ otherwise. In the two-agent setting, the two agents supervise each other. Crucially, each agent does not contribute to its own supervision target $\hat a_{-n}(x)$.

\textbf{Policy optimization.}
Each agent is trained with GRPO~\citep{shao2024deepseekmath}, while the proposed design remains optimizer-agnostic and can be incorporated into other RL frameworks. For each agent, the rewards $\{r_n^k\}_{k=1}^{K}$ are normalized within its own rollout group to obtain group-relative advantages. 

Let $\boldsymbol\theta=(\theta_1,\ldots,\theta_N)$. The overall training objective is
{\eqsetup
\begin{equation}
\max_{\boldsymbol\theta}\;\mathcal J_{\mathrm{Co\text{-}RL}}(\boldsymbol\theta)
=\frac1N\sum_{n=1}^{N}\mathcal J_{\mathrm{GRPO}}
\bigl(\theta_n;\{y_n^k,r_n^k\}_{k=1}^{K}\bigr),
\label{eq:corl_objective}
\end{equation}}%
where $\mathcal J_{\mathrm{GRPO}}$ is the objective in Eq.~\eqref{eq:grpo} computed on the rollouts and rewards of agent $n$.

\begin{figure}[t]
\begin{minipage}[t]{0.49\textwidth}\vspace{0pt}\centering
\includegraphics[width=\linewidth]{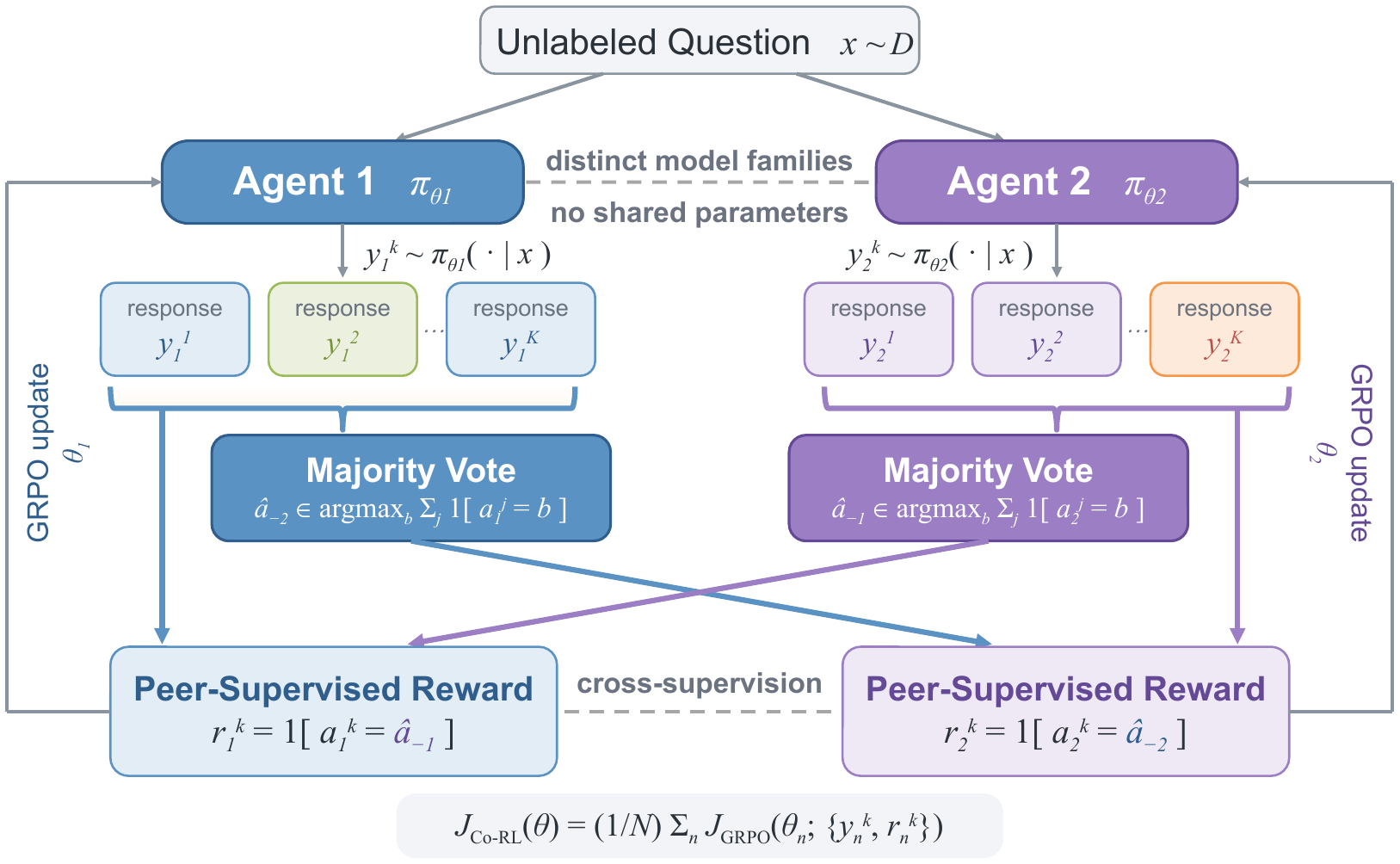}
\captionof{figure}{Training procedure of \method with two agents.}
\label{fig:overview}
\end{minipage}\hfill
\begin{minipage}[t]{0.49\textwidth}\vspace{0pt}
\hrule height 0.8pt\vspace{2pt}
\refstepcounter{algorithm}\label{alg:colearn}
\noindent\textbf{Algorithm~\thealgorithm}\ \method
\vspace{2pt}\hrule\vspace{2pt}
\small
\begin{algorithmic}[1]
\Require Agents $\{\pi_{\theta_n}\}_{n=1}^{N}$, unlabeled prompts $\mathcal D$
\For{each training step}
  \State Sample a batch $\mathcal B\subset\mathcal D$
  \For{$x\in\mathcal B$, $n=1,\ldots,N$}
    \State Sample $y_n^k\sim\pi_{\theta_n}(\cdot\mid x)$, $a_n^k\gets g(y_n^k)$
  \EndFor
  \For{$x\in\mathcal B$, $n=1,\ldots,N$}
    \State $\hat a_{-n}(x)\gets$ majority vote of $\{a_{n-1}^k\}_{k=1}^{K}$
    \State $r_n^k\gets\mathbf 1[a_n^k=\hat a_{-n}(x)]$ for $k=1,\ldots,K$
  \EndFor
  \State Update each $\theta_n$ with GRPO on $\{y_n^k,r_n^k\}_{k=1}^{K}$
\EndFor
\end{algorithmic}
\vspace{2pt}\hrule height 0.8pt
\end{minipage}
\end{figure}

At each training step, all rollouts and majority-vote pseudo-labels are computed before any policy update. Given the resulting rewards, each policy is then updated independently. Figure~\ref{fig:overview} and Algorithm~\ref{alg:colearn} give the complete procedure.

\subsection{Diverse Cohorts Enable Unsupervised Reasoning}

A central design principle of \method is that useful supervision should be both
\emph{decoupled} from the learner and \emph{complementary} to its own predictions.
Decoupling alone prevents an agent from directly constructing rewards from the same
predictions it is optimizing, but it does not guarantee that the resulting supervision
is informative. If two agents exhibit nearly identical behaviors, their errors can remain
highly correlated, causing the cohort to repeatedly reinforce the same incorrect answers.
The effectiveness of \method therefore depends not only on separating the learner from
its reward source, but also on promoting diversity among the agents that constitute the
cohort.

\label{sec:diversity}

\textbf{Decoupled supervision.}
Self-rewarding methods construct supervision from a model's own samples, coupling the
training target to the policy being optimized. This creates a self-reinforcing loop in
which incorrect predictions can become training targets and be further amplified. In
contrast, \method constructs each agent's reward exclusively from the other model, preventing
the agent from directly reinforcing its own output. This decoupling breaks the
self-confirmation loop without parameter sharing or gradient exchange. As shown in
Figure~\ref{fig:method_panels}(b,c), TTRL becomes unstable and eventually
collapses, whereas \method maintains stable pseudo-label quality and accuracy.

\textbf{Diversity across model families and scales.} Decoupling is most effective when the other agent makes complementary rather than correlated errors. We therefore construct heterogeneous cohorts by varying both model family and model scale. Distinct families differ throughout the model-development pipeline, including architecture, tokenization, pretraining data and post-training. Vision-language model families also use distinct vision encoders~\citep{bai2025qwenvl,chen2024internvl}. Figure~\ref{fig:method_panels}(d) shows that models from different families overlap less in their errors than models from the same family before training. During training, agents from the same family quickly reach high agreement, while agents from different families keep disagreeing on more problems (Figure~\ref{fig:method_panels}(a)). The disagreement gives each agent useful corrections and leads to higher accuracy for most of training (Figure~\ref{fig:method_panels}(c)). Varying the scale lets a smaller model train with a larger one and both still improve (Section~\ref{sec:w2s}). The weaker agent thus supervises the stronger one as in weak-to-strong supervision~\citep{pmlr-v235-burns24b}.

\begin{figure}[h]
  \centering
  \includegraphics[width=0.95\linewidth]{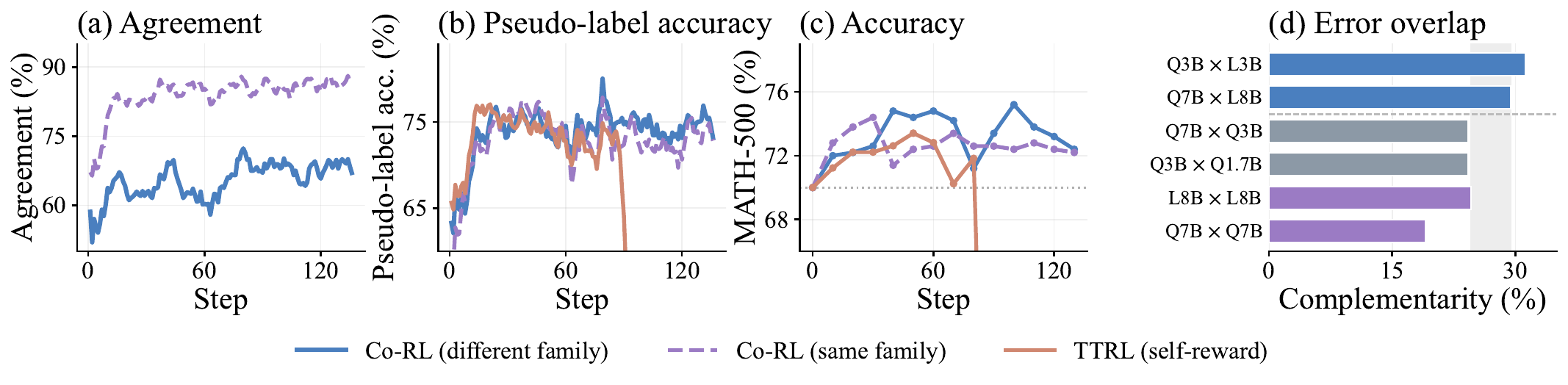}
  \caption{(a) Agreement rate between the majority votes of the two agents;
(b) pseudo-label accuracy; 
(c) evaluation performance; 
(d) error overlap for different base-model pairs before RL.}
  \label{fig:method_panels}
\end{figure}

\textbf{Diversity across input formulations.} Cohort diversity need not arise solely from model parameters. We further diversify agents by presenting different input formulations of the same underlying problem~\citep{corewarding2025}. For MATH, each training problem is rewritten using DeepSeek-R1~\citep{guo2025deepseekr1}; one agent is trained on the original prompt and the other on the rewritten version. These semantically equivalent but stylistically distinct inputs can induce different reasoning trajectories and reduce correlated reliance on the same lexical or structural cues, providing an additional source of complementary supervision. Examples are provided in Appendix~\ref{app:rephrase}.

\section{Theoretical Analysis}
\label{sec:theory}
In this section, we theoretically characterize the learning dynamics of
\method and compare them with self-rewarding GRPO. A key limitation of self-rewarding is
that an agent learns from a pseudo-label derived from its own predictions. As
a result, when the agent is systematically wrong, its supervision signal is
likely to reinforce the same error rather than correct it. In contrast,
\method derives supervision from another agent, allowing one agent's correct
prediction to provide a corrective signal for another agent's mistake.

We consider a fixed prompt $x$ and reduce the answer distribution to the
correct answer $a^\star$ and an aggregate incorrect answer. Let
$p_n=\Pr_{\pi_{\theta_n}}(a=a^\star\mid x)$ denote the probability that agent
$n$ assigns to the correct answer and in the single-agent setting we write $p$.
We analyze expected continuous-time updates of the logits
$\ell_n=\log(p_n/(1-p_n))$. Each agent draws the same odd number $K\ge3$ of
independent rollouts, and the two rollout groups are conditionally independent
given the current policies. We assume a common learning rate $\eta>0$ for the
symmetric two-agent analysis, while the general cross-agent dynamics
allow agent-specific rates $\eta_n>0$. We use normalized group-relative
advantages, assign zero update to constant-reward groups and omit clipping and
KL regularization.

\textbf{Training dynamics.}
The reward-induced dynamics under the two supervision schemes are as follows;
their derivation is deferred to Appendix~\ref{ap:prop1_comparison}.
\begin{proposition}[Training dynamics under self-rewarding and cross-agent supervision]
\label{prop:comparison_dynamics}
For self-rewarding, let $C\sim\operatorname{Bin}(K,p)$ count correct responses.
Then
\begin{equation}
\small
\dot p
=\eta p(1-p)
\mathbb E_{C\sim\operatorname{Bin}(K,p)}
\left[
\operatorname{sign}\!\left(C-\frac K2\right)
\frac{\sqrt{C(K-C)}}{K}
\right].
\label{eq:self_dynamics_main}
\end{equation}
For cross-agent supervision, let $C_n\sim\operatorname{Bin}(K,p_n)$ and let
$Z_{-n}=\mathbf1[\hat a_{-n}(x)=a^\star]$ indicate whether the pseudo-label from agent $n$'s designated peer is correct. Then
\begin{equation}
\small
\dot p_n
=\eta_n p_n(1-p_n)
\mathbb E\!\left[\operatorname{sign}\!\left(Z_{-n}-\frac12\right)\right]
\mathbb E_{C_n\sim\operatorname{Bin}(K,p_n)}
\left[\frac{\sqrt{C_n(K-C_n)}}{K}\right].
\label{eq:corl_dynamics_general_main}
\end{equation}
\end{proposition}

Let $V_K(p)=\sum_{c=(K+1)/2}^{K}\binom{K}{c}p^c(1-p)^{K-c}$ be the probability
that a majority vote is correct, and define $\phi_K(p)=2V_K(p)-1$. Let
$q_K(p)=\eta p(1-p)\mathbb E_{C\sim\operatorname{Bin}(K,p)}
[\sqrt{C(K-C)}/K]$, which is positive for $p\in(0,1)$. For two agents with
the same learning rate, Eq.~\eqref{eq:corl_dynamics_general_main} reduces to
\begin{equation}
\small
\dot p_A=q_K(p_A)\phi_K(p_B),
\qquad
\dot p_B=q_K(p_B)\phi_K(p_A).
\label{eq:corl_two_agent_main}
\end{equation}
Under self-rewarding, the pseudo-label comes from the agent's own rollouts.
The update is therefore \emph{self-confirming} and reinforces whichever answer
is currently more likely, whether correct or not. In contrast, \method
decouples supervision from the optimized agent. Since $q_K(p_A)>0$, agent
$A$'s update direction is determined entirely by the supervision signal
$\phi_K(p_B)$ from agent $B$. The corresponding
self-rewarding behavior is summarized below.

\begin{proposition}[Self-confirming dynamics]
\label{prop:self_confirming}
For odd $K\geq3$, define
{\small
\[
G_K^{\mathrm{self}}(p)
=
\mathbb E_{C\sim\operatorname{Bin}(K,p)}
\left[
\operatorname{sign}\!\left(C-\frac K2\right)
\frac{\sqrt{C(K-C)}}{K}
\right].
\]}%
For $p\in(0,1)$,
$\operatorname{sign}(G_K^{\mathrm{self}}(p))
=\operatorname{sign}(p-1/2)$.
Consequently, self-rewarding sends $p(0)<1/2$ to $p(t)\to0$ and
$p(0)>1/2$ to $p(t)\to1$.
\end{proposition}

\begin{theorem}[\method enlarges the basin of jointly correct convergence]
\label{thm:correct_basin}
Consider the symmetric two-agent dynamics in
Eq.~\eqref{eq:corl_two_agent_main} with interior initialization
$(p_A(0),p_B(0))\in(0,1)^2$. The correct and incorrect consensus states
$(1,1)$ and $(0,0)$ are asymptotically stable, while $(1/2,1/2)$ is a saddle
point whose interior separatrix is $p_A+p_B=1$. Consequently,
\[
\small
(p_A(t),p_B(t))\longrightarrow
\begin{cases}
(1,1), & p_A(0)+p_B(0)>1,\\
(0,0), & p_A(0)+p_B(0)<1,\\
(1/2,1/2), & p_A(0)+p_B(0)=1.
\end{cases}
\]
In particular, the basins where \emph{both} agents converge to the correct
answer satisfy
\[
\small
\underbrace{\{p_A>1/2,\ p_B>1/2\}}_{\mathcal B_{\mathrm{self}}^+}
\subsetneq
\underbrace{\{p_A+p_B>1\}}_{\mathcal B_{\mathrm{Co\text{-}RL}}^+},
\qquad (p_A,p_B)\in(0,1)^2.
\]
\end{theorem}
The proofs of Proposition~\ref{prop:self_confirming} and Theorem~\ref{thm:correct_basin} are in Appendices~\ref{ap:prop_self} and~\ref{ap:theorem1}.

\begin{takeaway}
\textcolor{takeawaygreen}{\textbf{Takeaway:}}
\textbf{\method uses complementary strengths to expand the basin of correct convergence.}
Consider an extreme example where $(p_A,p_B)=(0.9,0.2)$ on half of the prompts and $(p_A,p_B)=(0.2,0.9)$ on the other half. Each agent has average accuracy $0.55$, but their strengths are perfectly complementary. Under self-rewarding, each agent succeeds only on the half of prompts where $p>1/2$, yielding a final accuracy of $0.5$. In contrast, under \method, $p_A+p_B=1.1>1$ on every prompt, so Theorem~\ref{thm:correct_basin} predicts that both agents converge to the correct answer on all prompts, achieving accuracy $1.0$. 
\end{takeaway}

\definecolor{cqwen}{RGB}{123,80,162}     
\definecolor{cllama}{RGB}{0,114,178}     
\definecolor{cgemma}{RGB}{230,159,0}     
\definecolor{cintern}{RGB}{0,158,115}    
\definecolor{cneutral}{RGB}{90,110,132}  
\definecolor{ccomas}{RGB}{205,60,66}     

\definecolor{cqwen}{RGB}{123,80,162}    
\definecolor{cllama}{RGB}{0,114,178}    
\definecolor{ccomas}{RGB}{205,60,66}    
\definecolor{corange}{RGB}{232,126,42}  
\definecolor{cgain}{RGB}{20,120,50}     
\definecolor{cdrop}{RGB}{190,55,55}     

\section{Experiments}
\label{sec:experiments}
\subsection{Experimental Setup}

\textbf{Data and models.} Language models are trained on the level 3 to 5 split of MATH~\citep{hendrycks2021math}, following MARTI~\citep{marti2026}. We train Qwen2.5-3B with Llama-3.2-3B-Instruct and Qwen2.5-7B with Llama-3.1-8B-Instruct~\citep{yang2024qwen25,dubey2024llama3}, and add Qwen3-1.7B~\citep{qwen3}, Phi-4-mini-Instruct~\citep{phi4mini} or Granite-3.3-2B-Instruct~\citep{granite33} as a third agent in the three-agent experiments. Vision-language models are trained on MMR1-Math~\citep{mmr1_2025}, the training data of MM-UPT~\citep{mmupt2025}, and on multimodal-open-r1~\citep{multimodal_openr1_2025} as a second dataset. We train Qwen2.5-VL-3B with InternVL3.5-2B and Qwen2.5-VL-7B with InternVL3.5-8B~\citep{bai2025qwenvl,chen2024internvl}.



\textbf{Baselines.} On language models, we compare our method against
previous label-free self-rewarding methods, including TTRL
\citep{zuo2025ttrl}, Intuitor \citep{intuitor2025}, RENT \citep{rent2025}
and Co-rewarding-II \citep{corewarding2025}. On vision-language models, we compare with the
recent self-rewarding method MM-UPT \citep{mmupt2025}, which adopts the TTRL
reward formulation. For multi-agent RL baselines, we compare against MAPoRL~\citep{maporl2025} and CoMAS~\citep{comas2025}. GRPO trained with ground-truth rewards (GT-Reward) serves as a supervised reference~\citep{shao2024deepseekmath}.



\textbf{Training and evaluation.} Language models follow the training recipe of Co-rewarding~\citep{corewarding2025} with $K=12$ responses per prompt, and vision-language models follow R1-V~\citep{chen2025r1v} with $K=8$. Language models are evaluated on seven benchmarks in mathematics (GSM8K~\citep{cobbe2021gsm8k}, MATH-500~\citep{lightman2023letsverify}, AMC~\citep{aimo2024amc}), code (HumanEval~\citep{chen2021humaneval}, MBPP~\citep{austin2021mbpp}, LiveCodeBench~\citep{jain2024livecodebench}) and science (GPQA~\citep{rein2024gpqa}). Vision-language models are evaluated on MathVision~\citep{wang2024mathvision}, MathVerse~\citep{zhang2024mathverse}, MathVista~\citep{lu2024mathvista} and We-Math~\citep{qiao2024wemath}. Appendix~\ref{ap:exp} gives the full hyperparameters, evaluation settings and results.



\subsection{Results}
\label{sec:main}

\textbf{Comparison with prior self-reward methods and GT baseline.} Table~\ref{tab:llm3b} compares \textsc{Co-RL+}, which combines diversity across model families and input formulations by jointly training two models from different families while presenting one with rephrased questions, with base model, four prior self-rewarding methods and GT supervised RL. Without ground-truth supervision, \textsc{Co-RL+} improves the base models by 3.0--7.9\%, outperforming the best self-rewarding baseline across all models by up to 2.0\% and surpassing GT-Reward on all four models. 

\textbf{Out-of-distribution generalization.} 
Despite being trained exclusively on MATH, it also achieves better out-of-distribution performance on coding and science, achieving the best LiveCodeBench results for three models and the best GPQA accuracy on Qwen2.5-7B and Llama-3.1-8B-Instruct.

\definecolor{cavg}{RGB}{221,235,247}   
\definecolor{cours}{RGB}{238,238,238}
\begin{table*}[t]\centering
\setlength{\abovecaptionskip}{3pt}
\caption{\small Full performance across seven benchmarks for 3B, 7B and 8B models (\%). For each benchmark, the best label-free result is shown in \textbf{bold} and the second best is \underline{underlined}, with ties sharing the marking. Base and GT-Reward are excluded from the ranking. \textsc{Co-RL+} uses models from different families and trains one of the two agents on rephrased questions. Qwen2.5-3B is trained with Llama-3.2-3B-Instruct and Qwen2.5-7B is trained with Llama-3.1-8B-Instruct.}
\label{tab:llm3b}
\vspace{6pt}
\begingroup
\setlength{\tabcolsep}{6pt}\renewcommand{\arraystretch}{1.1}\small
\setlength{\aboverulesep}{1.5pt}\setlength{\belowrulesep}{1.5pt}%
\begin{tabular}{l *{8}{c}}
\toprule
\textbf{Method} & \textbf{GSM8K} & \textbf{MATH500} & \textbf{AMC} & \textbf{HEval} & \textbf{GPQA} & \textbf{MBPP} & \textbf{LCB} & \cellcolor{cavg}\textbf{Avg} \\
\midrule
\multicolumn{9}{c}{\textbf{\textit{Qwen2.5-3B}}}\\
\midrule
Base & 71.4 & 55.8 & 25.9 & 56.1 & 21.2 & 50.8 & 13.7 & \cellcolor{cavg}42.1 \\
GT-Reward & 76.2 & 64.6 & 36.1 & 65.2 & 20.7 & 54.4 & 14.5 & \cellcolor{cavg}47.4 \\
\cmidrule(lr){1-9}
TTRL & \underline{80.4} & \underline{66.4} & \underline{31.3} & \textbf{63.4} & 22.2 & 51.8 & 15.9 & \cellcolor{cavg}\underline{47.3} \\
RENT & 75.6 & 62.8 & \underline{31.3} & 59.2 & 18.2 & 52.4 & 14.5 & \cellcolor{cavg}44.9 \\
Intuitor & 74.9 & 64.2 & 26.5 & 59.8 & \textbf{27.3} & 50.4 & \underline{16.4} & \cellcolor{cavg}45.6 \\
Co-rewarding-II & 75.5 & 63.4 & 30.1 & 61.0 & 24.8 & \underline{53.2} & 11.0 & \cellcolor{cavg}45.6 \\
\rowcolor{cours}\textsc{Co-RL+} & \textbf{81.0} & \textbf{66.6} & \textbf{36.1} & \underline{62.8} & \underline{25.8} & \textbf{55.6} & \textbf{17.2} & \cellcolor{cavg}\textbf{49.3} \\
\midrule
\multicolumn{9}{c}{\textbf{\textit{Llama-3.2-3B-Instruct}}}\\
\midrule
Base & 73.6 & 43.8 & 18.1 & 51.2 & 21.2 & 50.8 & 12.0 & \cellcolor{cavg}38.7 \\
GT-Reward & 78.8 & 53.8 & 25.3 & 60.4 & 20.7 & 50.2 & 12.1 & \cellcolor{cavg}43.0 \\
\cmidrule(lr){1-9}
TTRL & \underline{77.9} & 50.2 & \underline{26.5} & \textbf{59.2} & \textbf{24.8} & \underline{51.2} & \underline{12.0} & \cellcolor{cavg}\underline{43.1} \\
RENT & 75.4 & 45.2 & 12.0 & \textbf{59.2} & 17.7 & 49.4 & 11.5 & \cellcolor{cavg}38.6 \\
Intuitor & 75.8 & 40.8 & 21.7 & 54.3 & 21.7 & \textbf{51.4} & \underline{12.0} & \cellcolor{cavg}39.7 \\
Co-rewarding-II & 75.4 & \underline{53.4} & 24.1 & \underline{54.9} & \underline{23.7} & 49.2 & \textbf{12.1} & \cellcolor{cavg}41.8 \\
\rowcolor{cours}\textsc{Co-RL+} & \textbf{78.4} & \textbf{55.2} & \textbf{30.1} & \textbf{59.2} & 22.2 & 50.4 & \underline{12.0} & \cellcolor{cavg}\textbf{43.9} \\
\midrule
\multicolumn{9}{c}{\textbf{\textit{Qwen2.5-7B}}}\\
\midrule
Base & 78.8 & 58.6 & 32.4 & 73.4 & 25.8 & 62.8 & 21.1 & \cellcolor{cavg}50.4 \\
GT-Reward & 88.8 & 76.8 & 45.6 & 80.7 & 24.2 & 64.2 & 25.0 & \cellcolor{cavg}57.9 \\
\cmidrule(lr){1-9}
TTRL & \underline{86.8} & 74.4 & 38.7 & 75.6 & 23.7 & 60.6 & 24.3 & \cellcolor{cavg}54.9 \\
RENT & 83.2 & 74.0 & \textbf{44.4} & 75.4 & \underline{30.3} & \underline{64.2} & \underline{25.3} & \cellcolor{cavg}\underline{56.7} \\
Intuitor & 84.1 & 72.6 & 40.5 & \underline{78.9} & 29.3 & \textbf{64.4} & 24.3 & \cellcolor{cavg}56.3 \\
Co-rewarding-II & 84.8 & \underline{76.2} & 43.2 & 77.8 & 29.8 & 63.0 & 15.4 & \cellcolor{cavg}55.8 \\
\rowcolor{cours}\textsc{Co-RL+} & \textbf{87.7} & \textbf{76.4} & \underline{43.7} & \textbf{79.7} & \textbf{31.3} & 63.0 & \textbf{26.2} & \cellcolor{cavg}\textbf{58.3} \\
\midrule
\multicolumn{9}{c}{\textbf{\textit{Llama-3.1-8B-Instruct}}}\\
\midrule
Base & 82.9 & 49.6 & 18.1 & 65.2 & 22.2 & 58.4 & 16.8 & \cellcolor{cavg}44.7 \\
GT-Reward & 82.7 & 53.2 & 25.3 & 64.0 & 30.3 & 59.2 & 15.2 & \cellcolor{cavg}47.1 \\
\cmidrule(lr){1-9}
TTRL & 83.9 & 51.0 & \textbf{27.7} & 64.6 & 21.2 & 58.2 & 16.3 & \cellcolor{cavg}46.1 \\
RENT & 79.5 & 48.2 & 21.7 & \textbf{67.7} & 19.7 & \textbf{60.0} & 16.0 & \cellcolor{cavg}44.7 \\
Intuitor & 79.7 & 45.8 & 21.7 & 65.8 & \underline{26.8} & 58.0 & 16.1 & \cellcolor{cavg}44.8 \\
Co-rewarding-II & \underline{84.7} & \underline{52.0} & 24.1 & \underline{67.1} & 22.2 & \underline{59.8} & \underline{16.5} & \cellcolor{cavg}\underline{46.6} \\
\rowcolor{cours}\textsc{Co-RL+} & \textbf{85.4} & \textbf{55.6} & \underline{26.5} & 64.6 & \textbf{27.3} & 57.2 & \textbf{17.1} & \cellcolor{cavg}\textbf{47.7} \\
\bottomrule
\end{tabular}
\endgroup
\end{table*}

\textbf{\method Outperforms multi-agent RL methods.}
Prior state-of-the-art multi-agent RL methods typically rely on an external judge, such as an LLM judge in CoMAS or a learned reward model in MAPoRL. For a fair comparison, we use Qwen2.5-3B-Instruct following the same setup as CoMAS. We report the CoMAS results as presented in the original paper. As shown in Table~\ref{tab:comas}, \method surpasses CoMAS by 4.0\% on average and by 12.8\% on MATH-500. It leads on five of the seven benchmarks while using half as many agents and no external reward models.

\textbf{Extending to Vision-Language Models.} VLM families differ in both their visual encoders and language backbones, providing better heterogeneity for \method. We train Qwen2.5-VL-3B with InternVL3.5-2B and Qwen2.5-VL-7B with InternVL3.5-8B on each of the two multimodal datasets (Table~\ref{tab:mllm_large}). Without ground-truth supervision, \method improves the base models by up to 4.4\%, outperforms TTRL in six of eight settings with a maximum gain of 6.5\% on Qwen2.5-VL-7B, and matches or surpasses GT-Reward in all four 7B/8B settings.

\begin{table*}[t]\centering
\caption{\small Vision-language models trained separately on open-r1 and MMR1 (\%). Qwen2.5-VL-3B is trained with InternVL3.5-2B and Qwen2.5-VL-7B with InternVL3.5-8B. Base and GT-Reward are excluded from the ranking. All \method rows use the Different family setting.}
\label{tab:mllm_large}
\vspace{6pt}
\setlength{\tabcolsep}{2pt}\renewcommand{\arraystretch}{1.1}\small
\begin{tabular}{lll ccccc}
\toprule
\textbf{Backbone} & \textbf{Data} & \textbf{Method} & \textbf{MathVision} & \textbf{MathVerse} & \textbf{MathVista} & \textbf{We-Math} & \cellcolor{cavg}\textbf{Avg} \\
\midrule
\multirow{8}{*}{\textbf{Qwen2.5-VL-3B}}
 & \multirow{4}{*}{open-r1}
   & GT-Reward & 24.37 & 40.10 & 51.70 & 54.83 & \cellcolor{cavg}42.75 \\
 & & Base & 23.72 & 37.34 & 44.90 & 52.87 & \cellcolor{cavg}39.71 \\
 & & TTRL & \underline{23.13} & \underline{38.43} & \underline{51.70} & \underline{56.44} & \cellcolor{cavg}\underline{42.42} \\
 & & \cellcolor{cours}\method & \cellcolor{cours}\textbf{23.65} & \cellcolor{cours}\textbf{38.81} & \cellcolor{cours}\textbf{53.70} & \cellcolor{cours}\textbf{57.36} & \cellcolor{cavg}\textbf{43.38} \\
\cmidrule(lr){2-8}
 & \multirow{4}{*}{MMR1}
   & GT-Reward & 24.51 & 39.87 & 53.00 & 58.91 & \cellcolor{cavg}44.07 \\
 & & Base & 23.72 & 37.34 & 44.90 & 52.87 & \cellcolor{cavg}39.71 \\
 & & TTRL & \underline{22.34} & \underline{36.42} & \underline{51.10} & \underline{55.63} & \cellcolor{cavg}\underline{41.37} \\
 & & \cellcolor{cours}\method & \cellcolor{cours}\textbf{23.09} & \cellcolor{cours}\textbf{38.20} & \cellcolor{cours}\textbf{52.20} & \cellcolor{cours}\textbf{58.05} & \cellcolor{cavg}\textbf{42.88} \\
\midrule
\multirow{8}{*}{\textbf{InternVL3.5-2B}}
 & \multirow{4}{*}{open-r1}
   & GT-Reward & 26.58 & 45.05 & 60.70 & 60.11 & \cellcolor{cavg}48.11 \\
 & & Base & 24.14 & 42.11 & 50.30 & 57.41 & \cellcolor{cavg}43.49 \\
 & & TTRL & \underline{25.66} & \textbf{44.44} & \underline{60.10} & \underline{59.71} & \cellcolor{cavg}\underline{47.48} \\
 & & \cellcolor{cours}\method & \cellcolor{cours}\textbf{26.64} & \cellcolor{cours}\underline{43.15} & \cellcolor{cours}\textbf{60.90} & \cellcolor{cours}\textbf{60.80} & \cellcolor{cavg}\textbf{47.87} \\
\cmidrule(lr){2-8}
 & \multirow{4}{*}{MMR1}
   & GT-Reward & 26.94 & 44.34 & 60.40 & 59.43 & \cellcolor{cavg}47.78 \\
 & & Base & 24.14 & 42.11 & 50.30 & 57.41 & \cellcolor{cavg}43.49 \\
 & & TTRL & \underline{26.35} & \underline{44.31} & \textbf{60.50} & \underline{60.57} & \cellcolor{cavg}\textbf{47.93} \\
 & & \cellcolor{cours}\method & \cellcolor{cours}\textbf{26.78} & \cellcolor{cours}\textbf{44.87} & \cellcolor{cours}\underline{58.70} & \cellcolor{cours}\textbf{60.69} & \cellcolor{cavg}\underline{47.76} \\
\midrule
\multirow{8}{*}{\textbf{Qwen2.5-VL-7B}}
 & \multirow{4}{*}{open-r1}
   & GT-Reward & 27.34 & 47.13 & 69.50 & 65.98 & \cellcolor{cavg}52.49 \\
 & & Base & 26.35 & 46.04 & 69.80 & 65.52 & \cellcolor{cavg}51.93 \\
 & & TTRL & \underline{17.01} & \underline{42.34} & \underline{67.10} & \underline{65.34} & \cellcolor{cavg}\underline{47.95} \\
 & & \cellcolor{cours}\method & \cellcolor{cours}\textbf{26.74} & \cellcolor{cours}\textbf{49.26} & \cellcolor{cours}\textbf{71.30} & \cellcolor{cours}\textbf{70.40} & \cellcolor{cavg}\textbf{54.42} \\
\cmidrule(lr){2-8}
 & \multirow{4}{*}{MMR1}
   & GT-Reward & 28.62 & 48.60 & 67.80 & 68.85 & \cellcolor{cavg}53.47 \\
 & & Base & 26.35 & 46.04 & 69.80 & 65.52 & \cellcolor{cavg}51.93 \\
 & & TTRL & \underline{27.01} & \underline{45.18} & \underline{70.70} & \underline{65.23} & \cellcolor{cavg}\underline{52.03} \\
 & & \cellcolor{cours}\method & \cellcolor{cours}\textbf{27.70} & \cellcolor{cours}\textbf{48.27} & \cellcolor{cours}\textbf{71.30} & \cellcolor{cours}\textbf{66.95} & \cellcolor{cavg}\textbf{53.56} \\
\midrule
\multirow{8}{*}{\textbf{InternVL3.5-8B}}
 & \multirow{4}{*}{open-r1}
   & GT-Reward & 59.70 & 70.81 & 77.60 & 83.74 & \cellcolor{cavg}72.96 \\
 & & Base & 57.66 & 68.78 & 77.10 & 82.18 & \cellcolor{cavg}71.43 \\
 & & TTRL & \underline{60.07} & \textbf{71.78} & \underline{77.20} & \textbf{83.74} & \cellcolor{cavg}\textbf{73.20} \\
 & & \cellcolor{cours}\method & \cellcolor{cours}\textbf{60.23} & \cellcolor{cours}\underline{70.74} & \cellcolor{cours}\textbf{78.40} & \cellcolor{cours}\underline{82.82} & \cellcolor{cavg}\underline{73.05} \\
\cmidrule(lr){2-8}
 & \multirow{4}{*}{MMR1}
   & GT-Reward & 62.43 & 71.80 & 77.50 & 84.14 & \cellcolor{cavg}73.97 \\
 & & Base & 57.66 & 68.78 & 77.10 & 82.18 & \cellcolor{cavg}71.43 \\
 & & TTRL & \textbf{62.01} & \underline{70.94} & \underline{77.80} & \underline{84.14} & \cellcolor{cavg}\underline{73.72} \\
 & & \cellcolor{cours}\method & \cellcolor{cours}\underline{61.38} & \cellcolor{cours}\textbf{71.52} & \cellcolor{cours}\textbf{78.80} & \cellcolor{cours}\textbf{85.11} & \cellcolor{cavg}\textbf{74.20} \\
\bottomrule
\end{tabular}
\end{table*}

\textbf{Fair comparison with MM-UPT.} To ensure a fair comparison with prior self-rewarding methods for VLMs, we follow the MM-UPT setup, using Qwen2.5-VL-7B as the base model and training on the same MMR1 data. We adopt their evaluation protocol and report the results in Table~\ref{tab:mmupt}, together with four additional self-improvement baselines. SRLM~\citep{yuan2024selfrewarding} uses the model itself as a judge to reward its responses, LMSI~\citep{huang2022lmsi} fine-tunes on majority-vote answers, and Genixer~\citep{zhao2024genixer} and STIC~\citep{deng2024stic} train on self-generated data. \method improves the base model by 5.7\% and outperforms all label-free methods on every benchmark. It exceeds MM-UPT by 2.0\% on average and by 3.8\% on MathVerse, while also surpassing ground-truth-supervised GRPO on average. Moreover, \method requires only about one quarter of MM-UPT's training budget.

\afterpage{%
\begin{table}[t]
\setlength{\abovecaptionskip}{3pt}
\begin{minipage}[t]{0.645\textwidth}\vspace{0pt}\centering
\setlength{\abovecaptionskip}{0pt}\captionof{table}{Comparison under the CoMAS setting (\%). All methods train on the same prompt mixture and are evaluated following the CoMAS protocol. The \method row uses the Different family setting.}
\label{tab:comas}
\vspace{4pt}
{\setlength{\tabcolsep}{1.5pt}\renewcommand{\arraystretch}{1.0}\fontsize{8pt}{9.5pt}\selectfont
\setlength{\aboverulesep}{1.5pt}\setlength{\belowrulesep}{1.5pt}
\scalebox{0.9}{%
\begin{tabular}{l *{8}{c}}
\toprule
\textbf{Method} & \textbf{GSM8K} & \textbf{MATH500} & \textbf{HEval} & \textbf{MBPP} & \textbf{MMLU} & \textbf{GPQA} & \textbf{SciBench} & \cellcolor{cavg}\textbf{Avg} \\
\midrule
Base & 85.40 & 55.00 & 73.78 & 55.80 & 63.20 & 28.79 & 36.47 & \cellcolor{cavg}56.92 \\
\cmidrule(lr){1-9}
MAPoRL & 85.80 & 55.40 & 75.61 & 57.00 & 63.20 & \textbf{31.47} & \textbf{39.08} & \cellcolor{cavg}58.22 \\
TTRL & \underline{88.20} & \underline{56.80} & 73.78 & 59.00 & 63.80 & 27.23 & \underline{38.48} & \cellcolor{cavg}58.18 \\
CoMAS & 87.20 & 55.80 & \underline{77.44} & \underline{59.20} & \underline{65.60} & \underline{29.69} & 37.68 & \cellcolor{cavg}\underline{58.94} \\
\rowcolor{cours}\method & \textbf{89.50} & \textbf{68.60} & \textbf{82.32} & \textbf{68.00} & \textbf{65.80} & \underline{29.69} & 36.87 & \cellcolor{cavg}\textbf{62.97} \\
\bottomrule
\end{tabular}}}
\end{minipage}\hfill
\begin{minipage}[t]{0.34\textwidth}\vspace{0pt}\centering
\includegraphics[width=\linewidth]{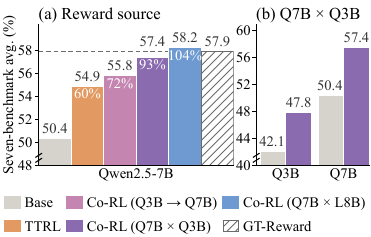}
\setlength{\abovecaptionskip}{3pt}\captionof{figure}{Weak-to-strong supervision (\%).}
\label{fig:w2s}
\end{minipage}
\vspace{-8pt}
\end{table}
\begin{table}[b]
\vspace{-8pt}
\setlength{\abovecaptionskip}{3pt}
\begin{minipage}[t]{0.535\textwidth}\vspace{0pt}\centering
\setlength{\abovecaptionskip}{0pt}
\captionof{table}{Comparison under the MM-UPT protocol (\%). Results for prior methods are reported from \citet{mmupt2025}. Base, SFT and GRPO are excluded from the ranking. All \method rows use the Different family setting.}
\label{tab:mmupt}
\vspace{4pt}
\setlength{\tabcolsep}{3.0pt}\renewcommand{\arraystretch}{1.0}\fontsize{8pt}{9.5pt}\selectfont
\setlength{\aboverulesep}{1.5pt}\setlength{\belowrulesep}{1.5pt}
\scalebox{0.9}{%
\begin{tabular}{l *{5}{c}}
\toprule
\textbf{Method} & \textbf{MathVision} & \textbf{MathVerse} & \textbf{MathVista} & \textbf{We-Math} & \cellcolor{cavg}\textbf{Avg} \\
\midrule
Base & 24.87 & 43.83 & 66.30 & 62.87 & \cellcolor{cavg}49.47 \\
GRPO & 29.01 & 45.03 & 71.40 & 67.24 & \cellcolor{cavg}53.17 \\
SFT & 26.45 & 43.53 & 63.30 & 64.20 & \cellcolor{cavg}49.37 \\
\cmidrule(lr){1-6}
SRLM & 25.33 & \underline{45.08} & 67.00 & 64.66 & \cellcolor{cavg}50.52 \\
LMSI & 24.83 & 43.76 & 64.90 & 66.38 & \cellcolor{cavg}49.97 \\
Genixer & 23.68 & 43.30 & 65.50 & 64.66 & \cellcolor{cavg}49.29 \\
STIC & 23.78 & 42.72 & 66.10 & 63.74 & \cellcolor{cavg}49.09 \\
MM-UPT & \underline{26.15} & 44.87 & \underline{72.90} & \underline{68.74} & \cellcolor{cavg}\underline{53.17} \\
\rowcolor{cours}\method & \textbf{28.06} & \textbf{48.68} & \textbf{73.30} & \textbf{70.80} & \cellcolor{cavg}\textbf{55.21} \\
\bottomrule
\end{tabular}}
\end{minipage}\hfill
\begin{minipage}[t]{0.45\textwidth}\vspace{0pt}\centering
\includegraphics[width=\linewidth]{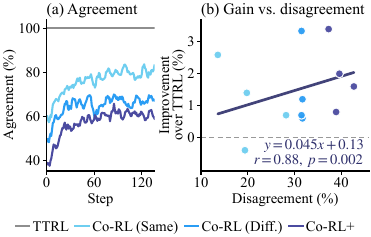}
\setlength{\abovecaptionskip}{3pt}\captionof{figure}{(a) Agreement during training of Qwen2.5-3B. (b) Improvement over TTRL against the final disagreement.}
\label{fig:diversity}
\end{minipage}
\end{table}
}

\subsection{Diversity by Decoupled Training, Model Families and Input Formulations}
\label{sec:diversity_family_input}

We first examine the effectiveness of the three sources of cohort diversity: decoupled training, model families and input formulations. Figure~\ref{fig:corl_variants} shows that average accuracy over the four LLMs rises as these sources are progressively introduced. Decoupled training already improves over self-rewarding baselines by allowing independently updated agents to supervise one another. Pairing models from different families further improves performance, demonstrating the benefit of heterogeneity across model families. Finally, \textsc{Co-RL+} introduces additional diversity through rephrased input formulations and achieves the best performance among the three settings.
A similar pattern is observed in the VLM experiments in Figure~\ref{fig:ensemble}, where \method (diff. family) consistently improves over TTRL.
Together, these results validate the effectiveness of all three design choices for promoting cohort diversity.

\subsection{Diversity across Model Scales and Weak-to-Strong Supervision}
\label{sec:w2s}
Model scale provides another natural source of cohort diversity. To test whether \method can co-train models of different sizes, we jointly train Qwen2.5-3B and Qwen2.5-7B without labels. As shown in Figure~\ref{fig:w2s}(b), both models improve over their base models: the 3B model increases from 42.1\% to 47.8\% and the 7B model from 50.4\% to 57.4\%. These results show that \method does not require both agents of comparable scale; models with substantially different capacities can provide useful supervision to each other and improve jointly.
Co-training models at different scales also yields a natural instance of weak-to-strong supervision, where a weaker model provides training signal to a stronger model~\citep{pmlr-v235-burns24b,lang2024theoretical}. Prior work on weak-to-strong generalization typically studies whether a strong student can learn from labels or predictions produced by a fixed weaker supervisor. In contrast, \method realizes weak-to-strong supervision through \emph{co-learning}: the weak and strong models supervise each other online and are optimized simultaneously, without ground-truth labels or a fixed teacher. The improvement of Qwen2.5-7B when paired with Qwen2.5-3B therefore shows that weak-to-strong supervision can emerge simply by jointly training models at different scales.

We further compare with simple weak-to-strong baseline where the Qwen2.5-3B is used as teacher to provide reward supervision for the RL training of Qwen2.5-7B. This further illustrates the effectiveness of decoupled-optimization. As is shown in Figure~\ref{fig:w2s}(a), Q3B $\rightarrow$ Q7B reaches only 55.8\%, outperforming TTRL (54.9\%) by about 1\%, but still below \method (57.4\%) by almost 2\%.
 We further compare with \method using  Llama-3.1-8B-Instruct and Qwen2.5-7B. Despite having more parameters, Llama-3.1-8B-Instruct achieves lower performance than Qwen2.5-7B (44.7\% vs.\ 50.4\%), providing another weak-to-strong setting. However, since Llama-3.1-8B-Instruct comes from a different family, it provides greater model heterogeneity beyond scale differences.
 As shown in Figure~\ref{fig:w2s}(a), Llama-3.1-8B-Instruct $\times$ Qwen2.5-7B further outperforms  Qwen2.5-7B $\times$ Qwen2.5-3B from 57.4\% to 58.2\%, surpassing the GT-Reward by 0.3\%. 
 This result suggests that the effectiveness of weak-to-strong extends beyond differences in model capability, but also depends on the complementarity of its errors, consistent with our broader finding that model heterogeneity can provide complementary supervision (see Appendix~\ref{app:decoupling}).



\subsection{Scaling up to More Agents in \method}
\label{sec:n3}
\method naturally extends beyond two agents, enabling multi-agent reinforcement learning with more cohorts. Due to computational constraints, we evaluate a three-agent setting as a proof of concept. Without loss of generality, we instantiate the three-agent cohort as a directed ring, where each agent is rewarded using the majority vote of its designated model (Section~\ref{sec:colearn}). We jointly train Qwen2.5-3B and Llama-3.2-3B-Instruct with a third model from a different family: Qwen3-1.7B, Phi-4-mini-Instruct, or Granite-3.3-2B-Instruct.

As shown in Appendix~\ref{app:n3}, three-agent \method improves all five base models, matches or exceeds GT-Reward on every model, and outperforms TTRL on four of five. Moreover, expanding the cohort from two to three agents can further improve performance, suggesting that additional agents can provide complementary supervision beyond the two-agent setting. These results demonstrate that \method can naturally scale to larger multi-agent cohorts.

\subsection{Effectiveness of the Diversity Designs}
\label{sec:ablation}

\begin{figure}[!h]
  \centering
  \includegraphics[width=\textwidth]{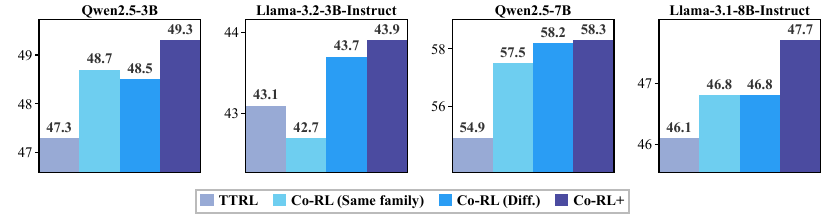}
  \caption{Seven-benchmark average (\%) of TTRL and the three \method variants.}
  \label{fig:corl_variants}
\end{figure}

We next examine whether the three diversity  designs in Section~\ref{sec:diversity} indeed increase diversity, and whether this diversity effectively leads to better performance.  We quantify diversity using the agreement rate between the majority-vote predictions of the two agents~\citep{li2023diverse}.

\textit{Each design increases diversity.}
As shown in Figure~\ref{fig:diversity}, each design consistently reduces agreement across all four models. By construction, TTRL has 100\% agreement, whereas the average agreement at the end of training decreases to 79.8\% for \method (same family), 68.4\% for \method (diff. family), and 60.4\% for \textsc{Co-RL+}. Model-family diversity is already present before training: for pairs from different families, exactly one model is correct on 29--31\% of the problems, compared with only 19--25\% for identical-model pairs from the same family (Figure~\ref{fig:method_panels}(d)).

\textit{Greater diversity leads to better performance.}
Average accuracy increases as additional sources of diversity are introduced, from 47.9\% for TTRL to 49.8\% for \textsc{Co-RL+} (Figure~\ref{fig:corl_variants}). Across the twelve runs, lower agreement is strongly associated with larger gains over TTRL (Figure~\ref{fig:diversity}(b)), with a repeated-measures correlation of $r=0.88$ ($p=0.002$)~\citep{bakdash2017rmcorr}.  These results illustrate that greater diversity leads to more improvements and better performance. 




\begin{figure}[t]
\setlength{\abovecaptionskip}{3pt}
\centering
\begin{minipage}[t]{0.49\textwidth}\vspace{0pt}\centering
\includegraphics[width=\linewidth]{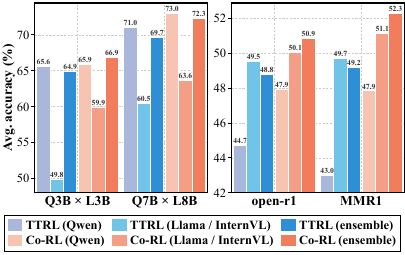}
\caption{Average accuracy (\%) of single models and their ensembles. The left panel shows Qwen2.5-3B with Llama-3.2-3B-Instruct and Qwen2.5-7B with Llama-3.1-8B-Instruct. The right panel shows Qwen2.5-VL-3B with InternVL3.5-2B trained on open-r1 and MMR1.}
\label{fig:ensemble}
\end{minipage}\hfill
\begin{minipage}[t]{0.49\textwidth}\vspace{0pt}\centering
\includegraphics[width=\linewidth]{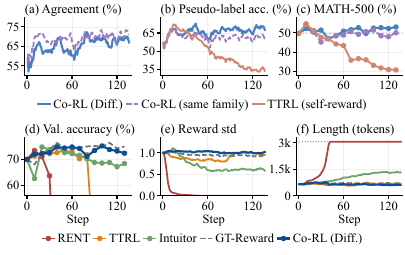}
\caption{(a)--(c) Llama-3.1-8B-Instruct trained with Qwen2.5-7B (different family), with a second Llama-3.1-8B-Instruct (same family) or alone with TTRL. (d)--(f) Training dynamics of Qwen2.5-7B. Reward standard deviation is normalized to its value at the first step.}
\label{fig:dyn_main}
\end{minipage}
\vspace{-8pt}
\end{figure}

\subsection{Training Dynamics and Stability}\label{sec:dynamics} 

Figure~\ref{fig:dyn_main}(d)--(f) examines the training dynamics of Qwen2.5-7B. All self-rewarding baselines exhibit instability over the course of training. RENT and Intuitor show progressively diminishing reward variance while generating excessively long responses, whereas TTRL suffers a sharp performance collapse in later stages. In contrast, \method maintains stable training dynamics throughout.

A key limitation of self-rewarding methods is that the same model both generates predictions and determines their rewards, creating a feedback loop in which policy updates directly influence subsequent supervision. This allows the model to reinforce its own predictions without necessarily improving their correctness. In TTRL, for instance, incorrect majority predictions are increasingly reinforced as training progresses, leading to a decline in pseudo-label accuracy (Figure~\ref{fig:dyn_main}(b)). \method mitigates this issue by decoupling reward generation from policy optimization: Qwen2.5-7B and Llama-3.1-8B-Instruct provide supervision to each other, reducing the direct dependence of each model's reward on its own predictions. This separation helps preserve reward variation (Figure~\ref{fig:dyn_main}(e)) and sustain pseudo-label accuracy (Figure~\ref{fig:dyn_main}(b)) throughout training. Consequently, \method avoids the late-stage degradation observed in self-rewarding baselines, achieving approximately 22 percentage points higher MATH-500 accuracy than TTRL for both models at the end of training.


\subsection{Fair Comparison under the Same Budgets} 

\method intrinsically involves two models, compared to self-rewarding methods such as TTRL, leading to twice the parameters and training computation. To isolate these effects, we conduct a budget-matched comparison against TTRL. Specifically, for each model pair used by \method, we train both corresponding base models independently with TTRL, thereby matching the set of optimized models and the overall training budget. We also match the inference budget by using eight rollouts per test example for both methods. For single-model evaluation, all eight rollouts are sampled from the same model, whereas for two-model ensembles, four rollouts are drawn from each model and aggregated.

As shown in Figure~\ref{fig:ensemble}, \method consistently outperforms the budget-matched TTRL counterparts at both the individual-model and ensemble levels. In particular, the \method ensemble exceeds the corresponding TTRL ensemble by 2.0--3.1\%. These controlled comparisons indicate that the gains of \method cannot be explained by additional optimization budget or test-time ensembling, but instead arise from the cross-agent supervision introduced during training. Full results are reported in Appendix~\ref{sec:ensemble_control}.

\section{Conclusion}
\label{sec:conclusion}

In this work, we introduced \method, a label-free multi-agent RL framework for reasoning tasks. In our framework, multiple agents learn from rewards constructed from the other agent's predictions rather than ground-truth labels or external judges. Across text-only and multimodal reasoning benchmarks, \method consistently improves diverse LLMs and VLMs, outperforming prior self-rewarding and multi-agent RL approaches and, in many settings, matching or surpassing training with ground-truth rewards. Our theoretical analysis shows that cross-agent supervision expands the set of initial conditions that converge to the correct solution, allowing \method to correct errors that self-rewarding RL would otherwise reinforce. An important direction for future work is to understand how the number, diversity, and interaction topology of agents shape cross-agent learning, and to develop adaptive supervision mechanisms that more effectively exploit complementary expertise.

\bibliographystyle{plainnat}
\bibliography{references}

\appendix

\section{Additional Optimization and Training Details}
\label{app:method_details}
\subsection{Complete GRPO Objective}
For completeness, we restate the GRPO objective of Eq.~\eqref{eq:grpo}. Let
the token-level importance ratio be
\begin{equation}
\rho_{k,t}(\theta)=
\frac{\pi_\theta(y_t^k\mid x,y_{<t}^k)}
     {\pi_{\theta_{\mathrm{old}}}(y_t^k\mid x,y_{<t}^k)}.
\label{eq:corl_importance_ratio}
\end{equation}
Writing the clipped surrogate as
$L_{k,t}(\theta)=\min\{\rho_{k,t}(\theta)\hat A^k,
\operatorname{clip}(\rho_{k,t}(\theta),1-\delta,1+\delta)\hat A^k\}$,
the objective is
\begin{equation}
\mathcal J_{\mathrm{GRPO}}(\theta)
=\mathbb E\!\left[
\frac1K\sum_{k=1}^{K}\frac1{|y^k|}
\sum_{t=1}^{|y^k|}
\bigl(L_{k,t}(\theta)-\beta\mathcal D_{\mathrm{KL}}^{k,t}\bigr)
\right].
\end{equation}
Here $\delta$ is the clipping threshold and $\beta$ weights deviation from the reference policy. In \method, each agent uses its own responses, peer-generated rewards, old policy, and reference policy in this objective. In \textsc{Co-RL+}, the agent trained on rephrased questions receives the rewritten question as its input $x_n$. 

\subsection{Additional Self-Rewarding Objectives}
For completeness, the Intuitor and RENT rewards of Section~\ref{sec:preliminary} are
\begin{align}
 r_{\mathrm{conf}}^k
 &=\frac1{|y^k|}\sum_{t=1}^{|y^k|}
 D_{\mathrm{KL}}\!\left(\mathbf U\,\Vert\,\pi_\theta(\cdot\mid x,y_{<t}^k)\right),
 &&\quad \text{Intuitor~\citep{intuitor2025}},\\
 r_{\mathrm{ent}}^k
 &=-\frac1{|y^k|}\sum_{t=1}^{|y^k|}
 \mathcal H\!\left(\pi_\theta(\cdot\mid x,y_{<t}^k)\right),
 &&\quad \text{RENT~\citep{rent2025}}.
\end{align}
Here $\mathbf U$ is uniform over the vocabulary and $\mathcal H$ is entropy. Each produces scalar rewards that can be normalized into advantages for Eq.~\eqref{eq:grpo}. Their differing training behavior is evaluated in Section~\ref{sec:dynamics} and Appendix~\ref{ap:training_dynamic}.

\section{Pre-training Error Decoupling}
\label{app:decoupling}

Section~\ref{sec:diversity} argues that the other agent helps when its errors
do not overlap with the agent's own and
Figure~\ref{fig:method_panels}(d) previews this overlap for three kinds of
model pairs. This appendix reports the full measurement. All numbers are
computed on base checkpoints before any RL, so they describe the
pretrained models themselves rather than anything our training produces.

\paragraph{Setup and metrics.} We score base checkpoints on 500 MATH
problems at levels 3 to 5, zero-shot and single-sample at $T{=}0.8$, with
rule-based extraction and equivalence checking. For a pair of models, every
problem lands in one of the four cells of Table~\ref{tab:contingency}, and
the four diversity measures \citep{kuncheva2003diversity} are counts over
these cells. Complementarity $c$ is the share of problems where exactly one
model is correct, so one model can correct the other. Oracle accuracy $u$
is the share of problems outside the both-wrong cell, which is the accuracy
a perfect selector would reach \citep{krogh1995ensemble}. Wrong-agreement
$w$ is the part of the both-wrong cell where the two models also return the
same answer, which is the case a majority vote cannot detect. Cohen's
$\kappa$ \citep{cohen1960kappa} measures how much more often than chance the two models are both correct or both wrong, so lower values mean more decoupled errors. All four numbers come from the same four cells. $\kappa$ and $c$
are therefore two readings of one measurement rather than independent
evidence, and across the twelve pairs below they correlate at $r=-0.98$.

\begin{table}[h]\centering
\caption{The four outcomes for a pair of models A and B. Every problem
falls into exactly one cell, and all four diversity measures are counts
over these cells.}
\label{tab:contingency}
\vspace{6pt}
\setlength{\tabcolsep}{10pt}\renewcommand{\arraystretch}{1.0}\fontsize{8pt}{9.5pt}\selectfont
\setlength{\aboverulesep}{1.5pt}\setlength{\belowrulesep}{1.5pt}
\begin{tabular}{l cc}
\toprule
 & B correct & B wrong \\
\midrule
A correct & both correct & only A correct \\
A wrong   & only B correct & both wrong \\
\bottomrule
\end{tabular}
\end{table}

\paragraph{Three levels of decoupling.} Table~\ref{tab:dec_pool} groups pairs by
what the two models differ in. \emph{Seed only} pairs a checkpoint with itself
under a different sampling seed, so the two views share every weight and differ
in generation noise alone. \emph{Same family} pairs models of one lineage across
sizes or generations. \emph{Different family} pairs models with separate
architectures and pretraining data.  Beyond the models of the main experiments, these pairs include Phi-3.5-mini~\citep{phi3report2024} and MiniCPM3-4B~\citep{minicpm2024}. The first level is what \method has by
construction, and the third is what a different-family cohort adds.

\paragraph{The three levels separate without overlap.} Every different-family
pair reaches $\kappa\le0.42$ and $c\ge29.4$, and every same-family and seed-only
pair sits at $\kappa\ge0.51$ and $c\le24.6$. No pair falls between the two
groups, at either scale. Averaged within a group, crossing families lowers
$\kappa$ from $0.53$ to $0.38$ and raises $c$ from $23.2$ to $30.8$. 

\paragraph{Holding the anchor fixed gives the same ordering.} Group averages mix
models of different strength, so we also fix one model and vary only the other model
(Table~\ref{tab:dec_anchor}). Both anchors give a monotone ladder. Relative to
pairing a model with itself, a same-family model lowers $\kappa$ by $0.04$ and
$0.07$, and a different-family model lowers it by $0.18$ and $0.16$ while
raising $c$ by 9.2\% and 10.4\%. Wrong-agreement is lowest with the Llama model under both anchors, though it does not order the other rows.


\begin{table}[t]\centering\small
\caption{Error decoupling before RL, by what the two models differ in, sorted by
$\kappa$ within each block.}
\label{tab:dec_pool}
\vspace{6pt}
\setlength{\tabcolsep}{7pt}
\begin{tabular}{llcccc}
\toprule
Decoupling & Pair & $\kappa\ \downarrow$ & $c\ \uparrow$ (\%) & $w\ \downarrow$ (\%) & $u\ \uparrow$ (\%) \\
\midrule
\multicolumn{6}{l}{\emph{3B tier}}\\
different family & Llama-3.2-3B $\times$ Phi-3.5-mini  & \textbf{0.31} & \textbf{32.8} & 3.0 & 53.0 \\
different family & Qwen2.5-3B $\times$ Llama-3.2-3B    & 0.38 & 31.2 & 2.4 & 63.0 \\
different family & Qwen2.5-3B $\times$ Phi-3.5-mini    & 0.38 & 31.2 & 4.0 & 55.4 \\
different family & Qwen2.5-3B $\times$ MiniCPM3-4B     & 0.41 & 29.4 & 4.4 & 60.4 \\
same family      & Qwen2.5-3B $\times$ Qwen3-1.7B-Base & 0.52 & 24.2 & 4.2 & 63.2 \\
seed only        & Qwen3-1.7B-Base $\times$ itself     & 0.52 & 24.0 & 5.0 & 66.4 \\
seed only        & Qwen2.5-3B $\times$ itself          & 0.56 & 22.0 & 4.4 & 62.6 \\
\midrule
\multicolumn{6}{l}{\emph{7B tier}}\\
different family & Qwen2.5-7B $\times$ Llama-3.1-8B    & \textbf{0.42} & \textbf{29.4} & 1.8 & 71.4 \\
same family      & Qwen2.5-7B $\times$ Qwen2.5-3B      & 0.51 & 24.2 & 3.8 & 69.8 \\
same family      & Qwen2.5-7B $\times$ Qwen3-1.7B-Base & 0.51 & 24.4 & 4.0 & 70.4 \\
seed only        & Llama-3.1-8B $\times$ itself        & 0.51 & 24.6 & 3.0 & 62.0 \\
seed only        & Qwen2.5-7B $\times$ itself          & 0.58 & 19.0 & 5.2 & 74.6 \\
\bottomrule
\end{tabular}
\end{table}

\begin{table}[t]\centering\small
\caption{One model held fixed and the other model varied.}
\label{tab:dec_anchor}
\vspace{6pt}
\setlength{\tabcolsep}{9pt}
\begin{tabular}{llccc}
\toprule
Peer & Decoupling & $\kappa\ \downarrow$ & $c\ \uparrow$ (\%) & $w\ \downarrow$ (\%) \\
\midrule
\multicolumn{5}{l}{\emph{Anchor: Qwen2.5-3B}}\\
itself, new seed & seed only        & 0.56 & 22.0 & 4.4 \\
Qwen3-1.7B-Base  & same family      & 0.52 & 24.2 & 4.2 \\
MiniCPM3-4B      & different family & 0.41 & 29.4 & 4.4 \\
Phi-3.5-mini     & different family & 0.38 & 31.2 & 4.0 \\
Llama-3.2-3B     & different family & \textbf{0.38} & \textbf{31.2} & \textbf{2.4} \\
\midrule
\multicolumn{5}{l}{\emph{Anchor: Qwen2.5-7B}}\\
itself, new seed & seed only        & 0.58 & 19.0 & 5.2 \\
Qwen2.5-3B       & same family      & 0.51 & 24.2 & 3.8 \\
Qwen3-1.7B-Base  & same family      & 0.51 & 24.4 & 4.0 \\
Llama-3.1-8B     & different family & \textbf{0.42} & \textbf{29.4} & \textbf{1.8} \\
\bottomrule
\end{tabular}
\end{table}

\clearpage

\section{Complete Proofs}
\subsection{Proof of Proposition~\ref{prop:comparison_dynamics}}
\label{ap:prop1_comparison}

In this section, we derive the reward-induced GRPO dynamics in
Proposition~\ref{prop:comparison_dynamics}. We first derive a common update
expression under a fixed pseudo-label and then specialize it to self-rewarding
and cross-agent supervision.

For a fixed prompt $x$, let
$X^k=\mathbf{1}[a^k=a^\star]$ indicate whether the $k$-th rollout is correct.
Under the binary reduction,
$X^k\overset{\mathrm{i.i.d.}}{\sim}\operatorname{Bernoulli}(p)$, and
$$
C=\sum_{k=1}^{K}X^k\sim\operatorname{Bin}(K,p)
$$
denotes the number of correct responses among the $K$ rollouts.
To analyze the induced binary dynamics, we use the log-odds
$$
\ell=\log\frac{p}{1-p}
$$
as a one-dimensional coordinate. Equivalently,
$p=\sigma(\ell)=1/(1+e^{-\ell})$.
We analyze the reward-induced component of GRPO in the infinitesimal-update
limit, where clipping is locally inactive.

\textbf{GRPO update under a fixed pseudo-label.}
Let $Z\in\{0,1\}$ indicate whether a fixed pseudo-label is correct, where
$Z=1$ corresponds to $a^\star$ and $Z=0$ to the aggregate incorrect answer.
The binary reward for rollout $k$ is
$$
r^k
=
\mathbf{1}[X^k=Z]
=
(1-Z)+(2Z-1)X^k.
$$
Since $C=\sum_{k=1}^{K}X^k$, the group-mean reward is
$$
\bar r
=
(1-Z)+(2Z-1)\frac{C}{K},
$$
and therefore
\begin{equation}
r^k-\bar r
=
(2Z-1)
\left(
X^k-\frac{C}{K}
\right).
\label{eq:centered_binary_reward}
\end{equation}

The within-group standard deviation of the binary rewards is
$$
s_K(C)
\triangleq
\sqrt{
\frac{1}{K}
\sum_{j=1}^{K}(r^j-\bar r)^2
}
=
\frac{\sqrt{C(K-C)}}{K}.
$$
For $0<C<K$, the normalized group-relative advantage is therefore
$$
\widehat A^k
=
\frac{r^k-\bar r}{s_K(C)}
=
\frac{2Z-1}{s_K(C)}
\left(
X^k-\frac{C}{K}
\right).
$$
For $C\in\{0,K\}$, all rewards are identical and the centered update is zero.

Treating the advantages as fixed during the policy update, the reward-induced
GRPO gradient in the log-odds coordinate is
$$
\begin{aligned}
g_\ell(C,Z)
&\triangleq
\frac{\partial}{\partial\ell}
\left[
\frac{1}{K}
\sum_{k=1}^{K}
\widehat A^k
\log\Pr_\ell(X^k)
\right] \\
&=
\frac{1}{K}
\sum_{k=1}^{K}
\widehat A^k
\frac{\partial}{\partial\ell}
\left[
X^k\log p+(1-X^k)\log(1-p)
\right] \\
&=
\frac{1}{K}
\sum_{k=1}^{K}
\widehat A^k(X^k-p) \\
&=
\frac{2Z-1}{Ks_K(C)}
\sum_{k=1}^{K}
\left(
X^k-\frac{C}{K}
\right)(X^k-p).
\end{aligned}
$$
Since
$\sum_{k=1}^{K}(X^k-C/K)=0$, the terms involving $p$ cancel. Using
$(X^k)^2=X^k$ and $\sum_{k=1}^{K}X^k=C$ gives
$$
\sum_{k=1}^{K}
\left(
X^k-\frac{C}{K}
\right)X^k
=
C-\frac{C^2}{K}
=
\frac{C(K-C)}{K}.
$$
Hence,
\begin{equation}
g_\ell(C,Z)
=
(2Z-1)
\frac{\sqrt{C(K-C)}}{K}.
\label{eq:fixed_label_update}
\end{equation}

\textbf{Self-rewarding.}
Under majority-vote self-rewarding, the pseudo-label is determined by the
same rollout group:
$$
Z_{\mathrm{self}}
=
\mathbf{1}\!\left[C>\frac{K}{2}\right].
$$
Since $K$ is odd,
$$
2Z_{\mathrm{self}}-1
=
\operatorname{sign}\!\left(C-\frac{K}{2}\right).
$$
Substituting into Eq.~\eqref{eq:fixed_label_update} gives
$$
g_\ell(C,Z_{\mathrm{self}})
=
\operatorname{sign}\!\left(C-\frac{K}{2}\right)
\frac{\sqrt{C(K-C)}}{K}.
$$
For a GRPO update with learning rate $\eta$,
$$
\ell^+-\ell
=
\eta\,g_\ell(C,Z_{\mathrm{self}}).
$$
Since $C\sim\operatorname{Bin}(K,p)$ is random, the conditional expected
one-step update is
$$
\mathbb{E}[\ell^+-\ell\mid p]
=
\eta\,
\mathbb{E}_{C\sim\operatorname{Bin}(K,p)}
\left[
\operatorname{sign}\!\left(C-\frac{K}{2}\right)
\frac{\sqrt{C(K-C)}}{K}
\right].
$$
In the infinitesimal-update limit, the corresponding mean log-odds dynamics
are
$$
\dot\ell
=
\eta\,
\mathbb{E}_{C\sim\operatorname{Bin}(K,p)}
\left[
\operatorname{sign}\!\left(C-\frac{K}{2}\right)
\frac{\sqrt{C(K-C)}}{K}
\right].
$$
Since $p=\sigma(\ell)$,
$$
\begin{aligned}
\dot p
&=
\frac{\partial p}{\partial\ell}\dot\ell \\
&=
\eta\,p(1-p)
\mathbb{E}_{C\sim\operatorname{Bin}(K,p)}
\left[
\operatorname{sign}\!\left(C-\frac{K}{2}\right)
\frac{\sqrt{C(K-C)}}{K}
\right],
\end{aligned}
$$
which recovers Eq.~\eqref{eq:self_dynamics_main}.

\textbf{\method: Cross-agent supervision.}
For agent $n$, let
$X_n^k=\mathbf{1}[a_n^k=a^\star]$ and
$$
C_n
=
\sum_{k=1}^{K}X_n^k
\sim\operatorname{Bin}(K,p_n).
$$
Let
$$
Z_{-n}
=
\mathbf{1}[\hat a_{-n}(x)=a^\star]
$$
indicate whether the pseudo-label constructed from the other agents is
correct. Conditional on the prompt and current policies, $Z_{-n}$ is
independent of agent $n$'s rollout group, and hence
$$
Z_{-n}\perp C_n.
$$

Applying Eq.~\eqref{eq:fixed_label_update} to agent $n$ gives
$$
g_{\ell_n}(C_n,Z_{-n})
=
\operatorname{sign}\!\left(Z_{-n}-\frac12\right)
\frac{\sqrt{C_n(K-C_n)}}{K},
$$
where
$\ell_n=\log\frac{p_n}{1-p_n}$.
Therefore,
$$
\begin{aligned}
\mathbb{E}[\ell_n^+-\ell_n\mid p_n,p_{-n}]
&=
\eta_n\,
\mathbb{E}
\left[
\operatorname{sign}\!\left(Z_{-n}-\frac12\right)
\frac{\sqrt{C_n(K-C_n)}}{K}
\right] \\
&=
\eta_n\,
\mathbb{E}
\left[
\operatorname{sign}\!\left(Z_{-n}-\frac12\right)
\right]
\mathbb{E}_{C_n\sim\operatorname{Bin}(K,p_n)}
\left[
\frac{\sqrt{C_n(K-C_n)}}{K}
\right],
\end{aligned}
$$
where the second equality follows from
$Z_{-n}\perp C_n$.

Taking the infinitesimal-update limit and using
$\partial p_n/\partial\ell_n=p_n(1-p_n)$ yields
$$
\begin{aligned}
\dot p_n
&=
\frac{\partial p_n}{\partial\ell_n}\dot\ell_n \\
&=
\eta_n p_n(1-p_n)
\mathbb{E}
\left[
\operatorname{sign}\!\left(Z_{-n}-\frac12\right)
\right]
\mathbb{E}_{C_n\sim\operatorname{Bin}(K,p_n)}
\left[
\frac{\sqrt{C_n(K-C_n)}}{K}
\right],
\end{aligned}
$$
which recovers Eq.~\eqref{eq:corl_dynamics_general_main}.

\textbf{Symmetric two-agent dynamics.}
We finally specialize the result to two symmetric agents $A$ and $B$.
For agent $A$, its pseudo-label is the majority vote of agent $B$'s
rollouts. Hence, if
$C_B\sim\operatorname{Bin}(K,p_B)$,
$$
Z_{-A}
=
\mathbf{1}\!\left[C_B>\frac{K}{2}\right],
$$
and therefore
$$
\begin{aligned}
\mathbb{E}
\left[
\operatorname{sign}\!\left(Z_{-A}-\frac12\right)
\right]
&=
\mathbb{E}_{C_B\sim\operatorname{Bin}(K,p_B)}
\left[
\operatorname{sign}\!\left(C_B-\frac{K}{2}\right)
\right] \\
&=
2V_K(p_B)-1
=
\phi_K(p_B).
\end{aligned}
$$
Similarly,
$$
\mathbb{E}
\left[
\operatorname{sign}\!\left(Z_{-B}-\frac12\right)
\right]
=
\phi_K(p_A).
$$

Assuming the same learning rate $\eta$ for both agents and defining
$$
q_K(p)
=
\eta p(1-p)
\mathbb{E}_{C\sim\operatorname{Bin}(K,p)}
\left[
\frac{\sqrt{C(K-C)}}{K}
\right]
>0,
$$
we obtain
$$
\dot p_A
=
q_K(p_A)\phi_K(p_B),
\qquad
\dot p_B
=
q_K(p_B)\phi_K(p_A),
$$
which recovers Eq.~\eqref{eq:corl_two_agent_main}.

\paragraph{Even $K$.} The analysis takes $K$ odd so that the majority vote has no ties. Our experiments use $K=12$ for language models and $K=8$ for vision-language models. A tie then goes to the answer sampled first, and because the rollouts are exchangeable this choice is uniform over the tied answers. For self-rewarding, a tie therefore yields the correct and the incorrect label with equal probability, and its expected contribution to Eq.~\eqref{eq:self_dynamics_main} is zero, matching $\mathrm{sign}(0)=0$. For cross-agent supervision, a tie only affects the accuracy of the peer label, and a majority vote over $2m$ rollouts with random tie-breaking is correct with the same probability as one over $2m-1$ rollouts.

\subsection{Proof of Proposition~\ref{prop:self_confirming}}
\label{ap:prop_self}

We prove that majority-vote self-rewarding reinforces the currently favored
answer and therefore induces self-confirming dynamics. For convenience, denote
$$
G_K^{\mathrm{self}}(p)
=
\mathbb{E}_{C\sim\operatorname{Bin}(K,p)}
\left[
\operatorname{sign}\!\left(C-\frac{K}{2}\right)
\frac{\sqrt{C(K-C)}}{K}
\right].
$$
Expanding the expectation gives
$$
G_K^{\mathrm{self}}(p)
=
\sum_{c=0}^{K}
\operatorname{sign}\!\left(c-\frac{K}{2}\right)
\frac{\sqrt{c(K-c)}}{K}
\Pr_p(C=c).
$$
Since the update magnitude satisfies
$$
\frac{\sqrt{c(K-c)}}{K}
=
\frac{\sqrt{(K-c)c}}{K},
$$
each event $C=c>K/2$ can be paired with its symmetric event
$C=K-c<K/2$. Moreover, the update magnitude vanishes at $c=0$ and $c=K$.
Hence,
\begin{equation}
G_K^{\mathrm{self}}(p)
=
\sum_{c=(K+1)/2}^{K-1}
\frac{\sqrt{c(K-c)}}{K}
\left[
\Pr_p(C=c)-\Pr_p(C=K-c)
\right].
\label{eq:self_pairing}
\end{equation}

For $p\in(0,1)$ and $c>K/2$,
$$
\begin{aligned}
\frac{\Pr_p(C=c)}
     {\Pr_p(C=K-c)}
&=
\frac{
\binom{K}{c}p^c(1-p)^{K-c}
}{
\binom{K}{K-c}p^{K-c}(1-p)^c
} \\
&=
\left(\frac{p}{1-p}\right)^{2c-K},
\end{aligned}
$$
where we use
$\binom{K}{c}=\binom{K}{K-c}$.
Since $2c-K>0$,
$$
\Pr_p(C=c)-\Pr_p(C=K-c)
\begin{cases}
>0, & p>\frac12,\\
=0, & p=\frac12,\\
<0, & p<\frac12.
\end{cases}
$$
The factor $\sqrt{c(K-c)}/K$ in
Eq.~\eqref{eq:self_pairing} is strictly positive for $0<c<K$.
Therefore, every term in the sum has the same sign, yielding
\begin{equation}
\operatorname{sign}\!\left(G_K^{\mathrm{self}}(p)\right)
=
\operatorname{sign}\!\left(p-\frac12\right).
\label{eq:self_sign}
\end{equation}

We next characterize the limiting behavior. From
Eq.~\eqref{eq:self_dynamics_main},
$$
\dot p
=
\eta p(1-p)G_K^{\mathrm{self}}(p).
$$
If $0<p(0)<1/2$, Eq.~\eqref{eq:self_sign} implies
$\dot p<0$ whenever $p\in(0,1/2)$. Thus, $p(t)$ is monotonically
decreasing and bounded below by zero, and therefore converges to some
$p_\infty\in[0,1/2)$.

Suppose $p_\infty>0$. By Eq.~\eqref{eq:self_sign},
$$
p_\infty(1-p_\infty)
G_K^{\mathrm{self}}(p_\infty)<0.
$$
Since the right-hand side of Eq.~\eqref{eq:self_dynamics_main} is continuous,
$\dot p$ remains strictly negative in a neighborhood of $p_\infty$, which
contradicts convergence to an interior limit. Hence,
$$
p(0)<\frac12
\quad\Longrightarrow\quad
p(t)\rightarrow0.
$$

Similarly, if $1/2<p(0)<1$, then $\dot p>0$. Thus, $p(t)$ is monotonically
increasing and bounded above by one. The same argument rules out any interior
limit, giving
$$
p(0)>\frac12
\quad\Longrightarrow\quad
p(t)\rightarrow1.
$$
Finally, at $p=1/2$, symmetry gives
$G_K^{\mathrm{self}}(1/2)=0$, so $p=1/2$ is an unstable equilibrium.
This completes the proof.

\subsection{Proof of Theorem~\ref{thm:correct_basin}}
\label{ap:theorem1}

We prove Theorem~\ref{thm:correct_basin} for the symmetric two-agent dynamics
from Eq.~\eqref{eq:corl_two_agent_main},
\begin{equation}
\dot p_A
=
q_K(p_A)\phi_K(p_B),
\qquad
\dot p_B
=
q_K(p_B)\phi_K(p_A),
\label{eq:corl_two_agent_appendix}
\end{equation}
where
$$
q_K(p)
=
\eta p(1-p)
\mathbb{E}_{C\sim\operatorname{Bin}(K,p)}
\left[
\frac{\sqrt{C(K-C)}}{K}
\right]
>0
$$
for $p\in(0,1)$, and
$\phi_K(p)=2V_K(p)-1$.
\begin{proof}
\textbf{Symmetry of the dynamics.}
We first establish two useful symmetries.
Since the distribution of $K-C$ under
$C\sim\operatorname{Bin}(K,p)$ is
$\operatorname{Bin}(K,1-p)$ and
$\sqrt{C(K-C)}$ is invariant under $C\mapsto K-C$, we have
$$
q_K(1-p)=q_K(p).
$$
Moreover, for odd $K$, complementing every rollout reverses the majority
outcome, so
$$
V_K(1-p)=1-V_K(p).
$$
Therefore,
\begin{equation}
q_K(1-p)=q_K(p),
\qquad
\phi_K(1-p)=-\phi_K(p).
\label{eq:corl_symmetry}
\end{equation}
Since $V_K(p)$ is strictly increasing and $V_K(1/2)=1/2$,
$$
\operatorname{sign}(\phi_K(p))
=
\operatorname{sign}\!\left(p-\frac12\right).
$$

\textbf{A conserved quantity.}
Define
\begin{equation}
F_K(p)
\triangleq
\int_{1/2}^{p}
\frac{\phi_K(u)}{q_K(u)}\,du.
\label{eq:corl_potential}
\end{equation}
Along any interior trajectory of
Eq.~\eqref{eq:corl_two_agent_appendix},
$$
\begin{aligned}
\frac{d}{dt}
\left[
F_K(p_A)-F_K(p_B)
\right]
&=
\frac{\phi_K(p_A)}{q_K(p_A)}\dot p_A
-
\frac{\phi_K(p_B)}{q_K(p_B)}\dot p_B \\
&=
\phi_K(p_A)\phi_K(p_B)
-
\phi_K(p_B)\phi_K(p_A) \\
&=0.
\end{aligned}
$$
Hence,
\begin{equation}
F_K(p_A)-F_K(p_B)
=
\text{constant}
\label{eq:corl_invariant}
\end{equation}
along every trajectory.

From Eq.~\eqref{eq:corl_symmetry},
$$
F_K(1-p)=F_K(p).
$$
Furthermore, $F_K'(p)=\phi_K(p)/q_K(p)$, so $F_K$ is strictly decreasing
on $(0,1/2)$ and strictly increasing on $(1/2,1)$, with
$F_K(1/2)=0$.

\textbf{Separatrix and basins of attraction.}
Consider first the disagreement region
$p_A<1/2<p_B$. In this region,
$$
\dot p_A>0,
\qquad
\dot p_B<0,
$$
so the two agents move toward one another.

Suppose $p_A+p_B>1$. Then
$p_B>1-p_A>1/2$. Since $F_K$ is strictly increasing on $(1/2,1)$ and
$F_K(1-p_A)=F_K(p_A)$,
$$
F_K(p_B)
>
F_K(1-p_A)
=
F_K(p_A),
$$
and therefore
\begin{equation}
F_K(p_A)-F_K(p_B)<0.
\label{eq:corl_correct_side}
\end{equation}
Because this quantity is conserved, agent $B$ cannot reach $1/2$ before
agent $A$: if $p_B=1/2$, then
$F_K(p_A)-F_K(p_B)=F_K(p_A)\geq0$, contradicting
Eq.~\eqref{eq:corl_correct_side}.
Thus, agent $A$ crosses the decision boundary first.

Once $p_A,p_B>1/2$,
Eq.~\eqref{eq:corl_two_agent_appendix} gives
$$
\dot p_A>0,
\qquad
\dot p_B>0.
$$
Both probabilities therefore increase monotonically and are bounded above by
one. No interior point of $(1/2,1)^2$ is an equilibrium because
$q_K(p)>0$ and $\phi_K(p)>0$ there. Hence,
$$
(p_A(t),p_B(t))
\longrightarrow
(1,1).
$$

Conversely, if $p_A+p_B<1$, then
$p_B<1-p_A$, and the same symmetry gives
$$
F_K(p_A)-F_K(p_B)>0.
$$
The conserved quantity now prevents agent $A$ from reaching $1/2$ before
agent $B$. Thus, $B$ crosses below $1/2$ first, after which both agents
satisfy
$$
\dot p_A<0,
\qquad
\dot p_B<0,
$$
and consequently
$$
(p_A(t),p_B(t))
\longrightarrow
(0,0).
$$
The case $p_B<1/2<p_A$ follows symmetrically.

Now consider $p_A+p_B=1$. Setting $p_B=1-p_A$ and using
Eq.~\eqref{eq:corl_symmetry},
$$
\begin{aligned}
\dot p_A+\dot p_B
&=
q_K(p_A)\phi_K(1-p_A)
+
q_K(1-p_A)\phi_K(p_A)\\
&=0.
\end{aligned}
$$
Hence, the line
\begin{equation}
p_A+p_B=1
\label{eq:corl_separatrix}
\end{equation}
is invariant. Along this line, if $p_A<1/2<p_B$, then
$\dot p_A>0$ and $\dot p_B<0$; the reverse holds when
$p_B<1/2<p_A$. Thus, every interior trajectory on this line converges to
$(1/2,1/2)$.

The cases where both agents initially lie on the same side of $1/2$ follow
directly from Eq.~\eqref{eq:corl_two_agent_appendix}. Combining all cases,
$$
p_A(0)+p_B(0)>1
\quad\Longrightarrow\quad
(p_A(t),p_B(t))\rightarrow(1,1),
$$
whereas
$$
p_A(0)+p_B(0)<1
\quad\Longrightarrow\quad
(p_A(t),p_B(t))\rightarrow(0,0).
$$
Therefore, $p_A+p_B=1$ is the interior separatrix between the two consensus
basins.

\textbf{Stability of the equilibria.}
In a neighborhood of $(1,1)$, both agents satisfy $p_A,p_B>1/2$ and hence
both coordinates increase monotonically toward one. Thus, $(1,1)$ is
asymptotically stable. By symmetry, $(0,0)$ is also asymptotically stable.

Finally, consider $(1/2,1/2)$. Since $\phi_K(1/2)=0$, let
$p_A=1/2+\delta_A$ and $p_B=1/2+\delta_B$. Linearizing
Eq.~\eqref{eq:corl_two_agent_appendix} gives
$$
\begin{bmatrix}
\dot\delta_A\\
\dot\delta_B
\end{bmatrix}
=
q_K(1/2)\phi_K'(1/2)
\begin{bmatrix}
0 & 1\\
1 & 0
\end{bmatrix}
\begin{bmatrix}
\delta_A\\
\delta_B
\end{bmatrix}.
$$
For odd $K$,
$$
\phi_K'(1/2)
=
2K
\binom{K-1}{(K-1)/2}
\left(\frac12\right)^{K-1}
>0.
$$
The Jacobian therefore has eigenvalues
$$
\lambda_{\pm}
=
\pm q_K(1/2)\phi_K'(1/2),
$$
with corresponding eigenvectors $(1,1)$ and $(1,-1)$.
Thus, $(1/2,1/2)$ has one unstable and one stable direction and is therefore
a saddle point. Its stable manifold is exactly the invariant line
$p_A+p_B=1$ established above.

Finally, Proposition~\ref{prop:self_confirming} gives the correct-convergence
basin under independent self-rewarding as
$$
\mathcal B_{\mathrm{self}}^+
=
\left\{
(p_A,p_B):
p_A>\frac12,\;
p_B>\frac12
\right\}.
$$
For \method, the result above gives
$$
\mathcal B_{\mathrm{\method}}^+
=
\left\{
(p_A,p_B):
p_A+p_B>1
\right\}.
$$
Therefore,
$$
\mathcal B_{\mathrm{self}}^+
\subsetneq
\mathcal B_{\mathrm{\method}}^+,
$$
which establishes the strictly larger basin of correct convergence.
\end{proof}

\section{Rephrased Training Questions}
\label{app:rephrase}

The \textsc{Co-RL+} runs train one agent on the original MATH questions and the other on a rephrased copy produced by DeepSeek-R1~\citep{guo2025deepseekr1}. The two copies are aligned by
row and the answer is never changed. The rewrites go beyond word substitution:
most of them place the problem in a concrete scenario, roughly doubling the
question length. In a sample of 300 pairs, every rewrite preserved the answer
and the row alignment. Two representative pairs are shown below.

\begin{rephrasebox}{Example 1 \hfill \emph{Answer:} 2}
{\scriptsize\textcolor{black!55}{ORIGINAL}}\par
How many vertical asymptotes does the graph of $y=\frac{2}{x^2+x-6}$ have?
\tcblower
{\scriptsize\textcolor{black!55}{REPHRASED}}\par
The function $f(t)=\frac{2}{t^2+t-6}$ describes the temperature of a chemical
reaction over time $t$. How many vertical asymptotes appear on the graph of
this function?
\end{rephrasebox}

\begin{rephrasebox}{Example 2 \hfill \emph{Answer:} $3\sqrt{3}$}
{\scriptsize\textcolor{black!55}{ORIGINAL}}\par
In triangle $ABC$, $AB=AC=14$ and $BC=26$. What is the length of the shortest
angle bisector in $ABC$? Express your answer in simplest radical form.
\tcblower
{\scriptsize\textcolor{black!55}{REPHRASED}}\par
A triangular park has two equal sides of 14 meters and a third side of 26
meters. The city plans a path from each corner that bisects its angle, and will
build only the shortest one. How long is that path? Express your answer in
simplest radical form.
\end{rephrasebox}

\section{Complete Experiment Results}\label{ap:exp}

\paragraph{Training hyperparameters.} All runs use the AdamW optimizer and sample rollouts at a temperature of 1.0. Following Co-rewarding~\citep{corewarding2025}, language models use a learning rate of $3\times10^{-6}$, an effective batch of 128 prompts per agent and a 3072-token cap, and train with $K=12$ responses per prompt for 2 epochs. The Intuitor baseline on Qwen2.5-7B uses a learning rate of $1\times10^{-6}$. Following R1-V~\citep{chen2025r1v}, vision-language models use a learning rate of $1\times10^{-6}$, a 1024-token cap and $K=8$ responses per prompt for 1 epoch.

\subsection{Three-agent \method}
\label{app:n3}
This experiment tests whether cross-agent supervision scales to three agents and to families other than Qwen and Llama. The three runs train Qwen2.5-3B and Llama-3.2-3B-Instruct together with a third model, which is Qwen3-1.7B, Phi-4-mini-Instruct or Granite-3.3-2B-Instruct. Each agent is rewarded by the majority vote of its designated peer, as in Section~\ref{sec:colearn}. All three runs share the training and evaluation setup of the two-agent language runs. Table~\ref{tab:abl_n3} compares each agent with the same model trained alone with GT-Reward or TTRL. The Qwen2.5-3B and Llama-3.2-3B-Instruct rows come from the run with Qwen3-1.7B, so they differ from Table~\ref{tab:llm3b}. For Granite-3.3-2B-Instruct, the TTRL row uses its best checkpoint before the collapse. \method improves all five base models and matches or exceeds GT-Reward on each.

\begin{table*}[t]\centering
\caption{Three-agent \method across seven benchmarks (\%).}
\label{tab:abl_n3}
\vspace{6pt}
\setlength{\tabcolsep}{4.0pt}\renewcommand{\arraystretch}{1.12}\small

\resizebox{\textwidth}{!}{\begin{tabular}{ll cccccccc}
\toprule
\textbf{Model} & \textbf{Method} & \textbf{GSM8K} & \textbf{MATH500} & \textbf{AMC} & \textbf{HEval} & \textbf{GPQA} & \textbf{MBPP} & \textbf{LCB} & \cellcolor{cavg}\textbf{Avg} \\
\midrule
\multirow{4}{*}{\textbf{Qwen2.5-3B}}
 & Base & 71.4 & 55.8 & 25.9 & 56.1 & 21.2 & 50.8 & 13.7 & \cellcolor{cavg}42.1 \\
 & GT-Reward & 76.2 & 64.6 & 36.1 & 65.2 & 20.7 & 54.4 & 14.5 & \cellcolor{cavg}47.4 \\
 & TTRL & \textbf{80.4} & \textbf{66.4} & \underline{31.3} & \underline{63.4} & \underline{22.2} & \underline{51.8} & \textbf{15.9} & \cellcolor{cavg}\underline{47.3} \\
 & \cellcolor{cours}\method & \cellcolor{cours}\underline{79.8} & \cellcolor{cours}\underline{66.3} & \cellcolor{cours}\textbf{33.6} & \cellcolor{cours}\textbf{64.6} & \cellcolor{cours}\textbf{23.2} & \cellcolor{cours}\textbf{56.0} & \cellcolor{cours}\underline{15.8} & \cellcolor{cavg}\textbf{48.5} \\
\midrule
\multirow{4}{*}{\textbf{Llama-3.2-3B-Instruct}}
 & Base & 73.6 & 43.8 & 18.1 & 51.2 & 21.2 & 50.8 & 12.0 & \cellcolor{cavg}38.7 \\
 & GT-Reward & 78.8 & 53.8 & 25.3 & 60.4 & 20.7 & 50.2 & 12.1 & \cellcolor{cavg}43.0 \\
 & TTRL & \textbf{77.9} & \underline{50.2} & \underline{26.5} & \underline{59.2} & \underline{24.8} & \textbf{51.2} & \textbf{12.0} & \cellcolor{cavg}\underline{43.1} \\
 & \cellcolor{cours}\method & \cellcolor{cours}\underline{77.8} & \cellcolor{cours}\textbf{54.2} & \cellcolor{cours}\textbf{28.8} & \cellcolor{cours}\textbf{64.4} & \cellcolor{cours}\textbf{25.1} & \cellcolor{cours}\underline{50.9} & \cellcolor{cours}\underline{11.7} & \cellcolor{cavg}\textbf{44.7} \\
\midrule
\multirow{4}{*}{\textbf{Qwen3-1.7B}}
 & Base & 67.0 & 60.9 & 27.5 & 40.0 & 15.3 & 50.6 & 12.4 & \cellcolor{cavg}39.1 \\
 & GT-Reward & 67.1 & 67.0 & 34.3 & 70.1 & 25.2 & 51.2 & 15.2 & \cellcolor{cavg}47.2 \\
 & TTRL & \textbf{70.3} & \textbf{67.6} & \underline{32.1} & \textbf{69.5} & \underline{24.8} & \underline{52.0} & \underline{15.1} & \cellcolor{cavg}\textbf{47.3} \\
 & \cellcolor{cours}\method & \cellcolor{cours}\underline{69.3} & \cellcolor{cours}\textbf{67.6} & \cellcolor{cours}\textbf{32.7} & \cellcolor{cours}\underline{64.2} & \cellcolor{cours}\textbf{27.1} & \cellcolor{cours}\textbf{54.6} & \cellcolor{cours}\textbf{15.8} & \cellcolor{cavg}\textbf{47.3} \\
\midrule
\multirow{4}{*}{\textbf{Phi-4-mini-Instruct}}
 & Base & 80.7 & 68.4 & 33.9 & 72.6 & 19.7 & 51.4 & 14.3 & \cellcolor{cavg}48.7 \\
 & GT-Reward & 84.9 & 74.4 & 36.9 & 73.8 & 21.2 & 54.8 & 13.5 & \cellcolor{cavg}51.4 \\
 & TTRL & \underline{84.6} & \textbf{72.8} & \underline{36.3} & \textbf{77.4} & \underline{23.7} & \underline{54.6} & \underline{10.0} & \cellcolor{cavg}\underline{51.4} \\
 & \cellcolor{cours}\method & \cellcolor{cours}\textbf{86.4} & \cellcolor{cours}\underline{71.8} & \cellcolor{cours}\textbf{38.9} & \cellcolor{cours}\underline{72.0} & \cellcolor{cours}\textbf{30.3} & \cellcolor{cours}\textbf{55.4} & \cellcolor{cours}\textbf{13.7} & \cellcolor{cavg}\textbf{52.6} \\
\midrule
\multirow{4}{*}{\textbf{Granite-3.3-2B-Instruct}}
 & Base & 63.5 & 54.6 & 24.1 & 32.9 & 25.8 & 45.0 & 10.3 & \cellcolor{cavg}36.6 \\
 & GT-Reward & 73.6 & 56.8 & 23.2 & 45.1 & 20.7 & 46.0 & 10.7 & \cellcolor{cavg}39.5 \\
 & TTRL & \underline{70.4} & \textbf{59.8} & \textbf{26.1} & \underline{34.8} & \underline{24.2} & \underline{45.6} & \underline{11.0} & \cellcolor{cavg}\underline{38.8} \\
 & \cellcolor{cours}\method & \cellcolor{cours}\textbf{72.9} & \cellcolor{cours}\underline{59.2} & \cellcolor{cours}\underline{23.2} & \cellcolor{cours}\textbf{42.1} & \cellcolor{cours}\textbf{26.8} & \cellcolor{cours}\textbf{46.8} & \cellcolor{cours}\textbf{11.2} & \cellcolor{cavg}\textbf{40.3} \\
\bottomrule
\end{tabular}}
\end{table*}

\subsection{Training Dynamics and Stability}
\label{ap:training_dynamic}

\begin{figure}[t]
\centering
\includegraphics[width=0.8\textwidth]{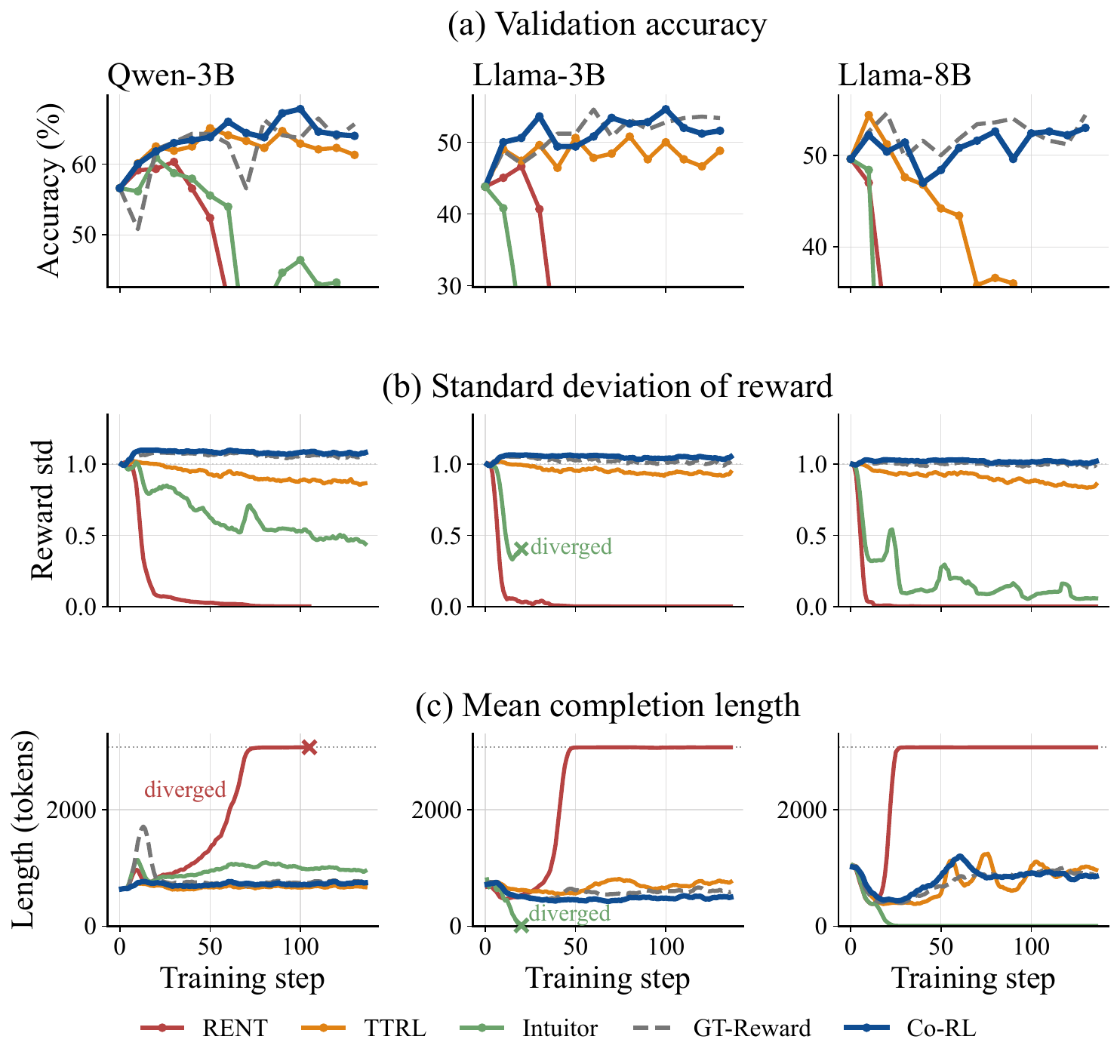}
\caption{Training dynamics of Qwen2.5-3B, Llama-3.2-3B-Instruct and Llama-3.1-8B-Instruct, one column per backbone. (a) MATH-500 validation accuracy. (b) Reward standard deviation relative to the first step. (c) Mean completion length. The dotted line marks the 3072-token cap. Crosses mark runs that diverged and stopped early. Figure~\ref{fig:dyn_main}(d)--(f) shows Qwen2.5-7B.}
\label{fig:dynamics}
\end{figure}

\paragraph{Text models.} Figure~\ref{fig:dynamics} shows validation accuracy, reward standard deviation and mean completion length for Qwen2.5-3B, Llama-3.2-3B-Instruct and Llama-3.1-8B-Instruct. Figure~\ref{fig:dyn_main}(d)--(f) in the main text shows Qwen2.5-7B. Across all four backbones, \method maintains a non-degenerate reward standard deviation and relatively stable completion lengths while steadily improving validation accuracy. In contrast, several self-rewarding methods become unstable during training. RENT rapidly drives the reward standard deviation toward zero and pushes completion length to the cap. Its accuracy then collapses. Intuitor similarly shows substantial reductions in reward variation and, for some models, degenerate completion lengths. TTRL keeps more reward variation, but its accuracy collapses on Qwen2.5-7B and declines steadily on Llama-3.1-8B-Instruct. On the two 3B models it stays below \method for most of training. These observations match the analysis in Section~\ref{sec:theory}, where supervision from a peer keeps the reward informative while self-generated supervision reinforces the agent's own errors.

\paragraph{Vision-language models.} We observe the same qualitative behavior in the multimodal setting. As shown in Figure~\ref{fig:mllm_dyn}, \method continues to improve evaluation accuracy while its completion length returns to the initial level, whereas TTRL eventually degrades in both accuracy and response length. More importantly, the agreement between the two agents remains well below full agreement throughout training, while the accuracy of the exchanged pseudo-labels steadily increases. The agents therefore do not converge to identical behavior. They keep meaningful differences while giving each other reliable supervision. This supports the mechanism predicted by our theory. Cross-agent learning benefits from complementary supervision and does not require the agents to collapse to the same predictions.

\begin{figure}[t]\centering
  \includegraphics[width=\textwidth]{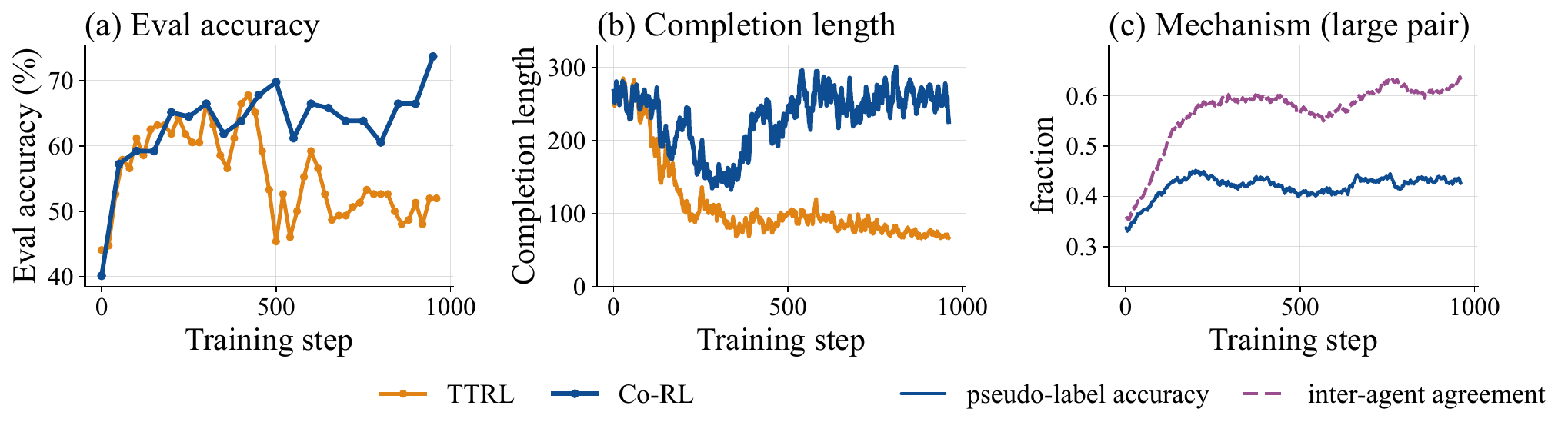}
  \caption{Training dynamics for Qwen2.5-VL-7B trained with InternVL3.5-8B on
open-r1. (a) Evaluation accuracy, (b) mean completion length, and (c) the
accuracy of the exchanged pseudo-labels together with the agreement between
the two agents. \method keeps improving and holds its completion length,
while TTRL peaks and then degrades in both.}
  \label{fig:mllm_dyn}
\end{figure}

\subsection{Controlling for Training and Test-Time Budgets}
\label{sec:ensemble_control}
\method trains two agents while a self-rewarding baseline trains one, so we test whether its gains come from cross-agent supervision rather than from the extra model or from ensembling. We train the same two base models independently with TTRL on the same prompts to match the training budget. At test time every prediction uses eight rollouts sampled at temperature 0.6 with top-$p$ 0.95. A single agent votes over eight of its own rollouts, and an ensemble pools four rollouts from each of its two models. One text pair is Qwen2.5-3B with Llama-3.2-3B-Instruct. The other is Qwen2.5-7B with Llama-3.1-8B-Instruct. Both are evaluated on GSM8K, MATH-500 and AMC, whose short final answers can be compared directly by majority voting. The vision-language pair is Qwen2.5-VL-3B with InternVL3.5-2B. It is trained separately on open-r1 and on MMR1 and evaluated on the four multimodal benchmarks.

Tables~\ref{tab:ens_llm} and~\ref{tab:ens_mllm} give the per-benchmark results, and Figure~\ref{fig:ensemble} plots their averages. These scores come from sampled votes, and under the matched budget \method leads TTRL for every single agent and every ensemble. Its advantage thus comes neither from training two models nor from ensembling alone.

\begin{table}[t]\centering\small
\caption{Matched-budget comparison between TTRL and \method on text
reasoning benchmarks. Each block trains the same two base models with both
methods. The ensemble rows pool four rollouts from each of the two models
for majority voting (maj@8, $T=0.6$). \textbf{Avg} is the macro-average over
the three benchmarks. Rows are compared within a block only.}
\label{tab:ens_llm}
\vspace{6pt}
\setlength{\tabcolsep}{7pt}\renewcommand{\arraystretch}{1.12}
\begin{tabular}{l ccc c}
\toprule
\textbf{Setting} & \textbf{GSM8K} & \textbf{MATH-500} & \textbf{AMC} & \cellcolor{cavg}\textbf{Avg} \\
\midrule
\multicolumn{5}{l}{\textbf{\textit{Qwen2.5-3B with Llama-3.2-3B-Instruct}}}\\
\cmidrule(lr){1-5}
TTRL (Qwen2.5-3B)      & \underline{88.2} & 68.8 & \textbf{39.8} & \cellcolor{cavg}65.6 \\
TTRL (Llama-3.2-3B)    & 65.7 & 56.0 & 27.7 & \cellcolor{cavg}49.8 \\
TTRL (ensemble)        & \underline{88.2} & 68.0 & \underline{38.6} & \cellcolor{cavg}64.9 \\
\rowcolor{cours}\method (Qwen2.5-3B)   & 87.4 & \textbf{72.8} & 37.4 & \cellcolor{cavg}\underline{65.9} \\
\rowcolor{cours}\method (Llama-3.2-3B) & 87.3 & 58.8 & 33.7 & \cellcolor{cavg}59.9 \\
\rowcolor{cours}\method (ensemble)     & \textbf{90.1} & \underline{70.8} & \textbf{39.8} & \cellcolor{cavg}\textbf{66.9} \\
\midrule
\multicolumn{5}{l}{\textbf{\textit{Qwen2.5-7B with Llama-3.1-8B-Instruct}}}\\
\cmidrule(lr){1-5}
TTRL (Qwen2.5-7B)      & 88.3 & \textbf{80.2} & 44.6 & \cellcolor{cavg}71.0 \\
TTRL (Llama-3.1-8B)    & 87.0 & 59.6 & 34.9 & \cellcolor{cavg}60.5 \\
TTRL (ensemble)        & 89.5 & 76.4 & 43.4 & \cellcolor{cavg}69.7 \\
\rowcolor{cours}\method (Qwen2.5-7B)   & \textbf{90.8} & \underline{80.0} & \underline{48.2} & \cellcolor{cavg}\textbf{73.0} \\
\rowcolor{cours}\method (Llama-3.1-8B) & 87.8 & 62.0 & 41.0 & \cellcolor{cavg}63.6 \\
\rowcolor{cours}\method (ensemble)     & \underline{90.6} & 76.8 & \textbf{49.4} & \cellcolor{cavg}\underline{72.3} \\
\bottomrule
\end{tabular}
\end{table}

\begin{table}[t]\centering\small
\caption{Matched-budget comparison between TTRL and \method on multimodal
reasoning benchmarks. Both settings train the same two base models,
Qwen2.5-VL-3B and InternVL3.5-2B. The ensemble rows pool four rollouts from
each of the two models for majority voting (maj@8, $T=0.6$, top-$p$ 0.95).
Rows are grouped by training set.}
\label{tab:ens_mllm}
\vspace{6pt}
\setlength{\tabcolsep}{5pt}\renewcommand{\arraystretch}{1.1}
\begin{tabular}{l cccc c}
\toprule
\textbf{Setting} & \textbf{MathVision} & \textbf{MathVerse} & \textbf{MathVista} & \textbf{We-Math} & \cellcolor{cavg}\textbf{Avg} \\
\midrule
\multicolumn{6}{l}{\textbf{\textit{open-r1}}}\\
\cmidrule(lr){1-6}
TTRL (Qwen2.5-VL)      & 22.96 & 31.45 & 61.10 & 63.39 & \cellcolor{cavg}44.73 \\
TTRL (InternVL3.5)     & \underline{29.67} & \textbf{38.91} & 62.30 & 67.24 & \cellcolor{cavg}49.53 \\
TTRL (ensemble)        & 27.24 & 35.13 & \underline{65.40} & 67.41 & \cellcolor{cavg}48.80 \\
\rowcolor{cours}\method (Qwen2.5-VL)   & 25.43 & 35.66 & 64.80 & 65.80 & \cellcolor{cavg}47.92 \\
\rowcolor{cours}\method (InternVL3.5)  & \textbf{30.46} & \underline{38.60} & 63.30 & \underline{67.87} & \cellcolor{cavg}\underline{50.06} \\
\rowcolor{cours}\method (ensemble)     & 28.95 & 38.48 & \textbf{67.00} & \textbf{69.08} & \cellcolor{cavg}\textbf{50.88} \\
\midrule
\multicolumn{6}{l}{\textbf{\textit{MMR1}}}\\
\cmidrule(lr){1-6}
TTRL (Qwen2.5-VL)      & 17.27 & 30.71 & 63.40 & 60.57 & \cellcolor{cavg}42.99 \\
TTRL (InternVL3.5)     & 28.78 & 39.47 & 63.70 & 66.90 & \cellcolor{cavg}49.71 \\
TTRL (ensemble)        & 25.53 & 37.77 & \underline{67.00} & 66.44 & \cellcolor{cavg}49.19 \\
\rowcolor{cours}\method (Qwen2.5-VL)   & 25.86 & 34.59 & 66.00 & 64.94 & \cellcolor{cavg}47.85 \\
\rowcolor{cours}\method (InternVL3.5)  & \textbf{30.79} & \textbf{40.94} & 65.30 & \underline{67.53} & \cellcolor{cavg}\underline{51.14} \\
\rowcolor{cours}\method (ensemble)     & \underline{30.49} & \underline{39.75} & \textbf{69.40} & \textbf{69.54} & \cellcolor{cavg}\textbf{52.30} \\
\bottomrule
\end{tabular}
\end{table}

\subsection{Evaluation Details}
\label{ap:setting}

\textbf{Language model evaluation.} All language benchmarks except
LiveCodeBench run through lm-evaluation-harness. Sampling uses a temperature
of 0.6 with top-$p$ 0.95 and a 3072-token generation budget. Trained
checkpoints are evaluated with their tokenizer's chat template, which
matches the prompt format used during training. For AMC we sample eight responses
per problem and report avg@8, averaged over three evaluation seeds. GPQA uses the Diamond subset with a
boxed-answer prompt. Answers are scored by rule-based graders with
math-aware normalization, and multiple-choice questions accept both the
option letter and the option value. LiveCodeBench (release v6) runs through
its official harness at its default temperature of 0.2 with one sample per
problem. The harness hard-codes a gated Llama-3 tokenizer for prompt
construction, and we substitute the openly available Llama-3.2-3B-Instruct
tokenizer, which carries the same chat template.

\textbf{CoMAS comparison.} The comparison with CoMAS uses its benchmark suite (GSM8K, MATH-500, HumanEval, MBPP, SciBench~\citep{wang2024scibench}, GPQA and MMLU~\citep{hendrycks2021mmlu}), its driver and its graders, with training prompts drawn from their blended 2,000 problems from MATH~\citep{hendrycks2021math}, KodCode~\citep{xu2025kodcode} and WebInstruct-verified~\citep{ma2025generalreasoner}. Their protocol draws
five samples per question at temperature 0.7 and then issues a sixth call
that reasons over the five drafts, and only the sixth response is graded.

\textbf{Vision-language evaluation.} We follow the benchmark splits of
MM-UPT, the MathVision test set (3{,}040 problems), the MathVerse testmini
split (3{,}940, all five versions), the MathVista testmini split (1{,}000)
and the We-Math testmini split (1{,}740). Decoding is greedy with a
16384-token generation budget. Images are resized so that the long side
does not exceed 1024 pixels, matching the training-time preprocessing.
Trained checkpoints are prompted in the format they were trained on, with
reasoning in think tags and the final answer in answer tags. Base models
receive a standard boxed-answer prompt. Scoring is two-stage. A rule-based
pass extracts the final answer and grades it with math-aware matching, and
a response that never commits to an answer in a recognized format counts as
incorrect. Responses that follow the format but fail the rule match are
passed to an LLM judge, Qwen2.5-32B-Instruct at temperature 0, which
accepts only semantically equivalent answers and rejects responses that are
cut off. 

\textbf{MM-UPT comparison.} Table~\ref{tab:mmupt} follows the evaluation protocol of MM-UPT so that the rows reported from their paper are comparable to ours. Responses are sampled at temperature 0.2 with top-$p$ 1.0, and answers are extracted and graded by Qwen2.5-32B-Instruct as a judge on the benchmark splits used in the MM-UPT paper. 

\textbf{Engineering notes.} All fixes below are included in our code.
None of them changes the training or evaluation semantics. They repair
crashes or a wrong backend choice in the underlying libraries.




\emph{Qwen2.5-VL, vision-tower attention backend.} In vLLM 0.11.2, the
helper that selects the vision tower's attention backend silently promotes
xFormers to the bundled FlashAttention build. That build supports head
dimensions that are multiples of 32 only, and the Qwen2.5-VL vision tower
has head dimension 80, so the model crashes at load. We patch the helper to
keep the original xFormers choice, which has no head-dimension restriction.
The InternVL vision tower has a head dimension that is a multiple of 32 and is unaffected, and later vLLM releases fix the bug.

\emph{InternVL3.5, processor and tiling.} We use the transformers-native HF
variants, whose checkpoints load through \texttt{AutoProcessor} without the
legacy remote-code path. Dynamic patch tiling is disabled at both training
and evaluation, so the image-token count per sample is identical in the two
settings. Model families are detected from each checkpoint's configuration
file rather than from directory names, since \method run directories contain the names of both models.

\end{document}